\documentclass[runningheads]{llncs}
\newcommand {\rf} {\mathit{rank}}

\newcommand {\lingconc} {\mathcal{S}}

\newcommand {\nott} {\lnot}

\newcommand {\sx} {\langle}
\newcommand {\dx} {\rangle}

\newcommand {\emme} {\mathcal{M}}
\newcommand {\enne} {\mathcal{N}}

\newcommand {\unione} {\cup}

\newcommand {\tc} {\mid}
\newcommand {\vuoto} {\emptyset}

\newcommand{\tip}{{\bf T}}

\newcommand{\alc}{\mathcal{ALC}}

\newcommand{\alctr}{\mathcal{ALC}+\tip_{\textsf{\tiny R}}}

\newcommand{\alctrbp}{{\mathcal{ALC}^{\Ra}\tip}_{BP} }

\newcommand{\be}{\begin{enumerate}}
\newcommand{\ee}{\end{enumerate}}

\newcommand{\hide}[1]{}

\def \cases{\left \{\begin{array}{l}}
\def \endcases{\end{array}\right .}

\newcommand {\Ra} {{\bf R}}

\newcommand {\bes} {\begin{description}}
\newcommand{\ens} {\end{description}}
\newcommand {\la} {\langle}
\newcommand {\ra} {\rangle}

\newcommand {\beq} {\begin{quote}}
\newcommand {\enq} {\end{quote}}
\newcommand {\bit} {\begin{itemize}}
\newcommand {\enit} {\end{itemize}}

\newenvironment{pozz}{\color{black}}{\color{black}}
\usepackage{times}
\usepackage{helvet}
\usepackage{courier}
\usepackage{times}
\usepackage{graphicx}
\usepackage{latexsym}
\usepackage{graphicx}
\usepackage{color}
\usepackage{amssymb}
\usepackage{colortbl}
\usepackage{amsmath}
\usepackage{microtype}%if unwanted, comment out or use option "draft"

\begin{document}
\bibliographystyle{plain}

\title{Reasoning about exceptions in ontologies: \\
from the lexicographic closure to the skeptical closure}
%a skeptical preferential approach (Extended Abstract)}

%\thanks{This research has been partially supported by INDAM - GNCS Project 2016 %{\em Ragionamento Defeasible nelle Logiche Descrittive}.
% {\em  Defeasible Reasoning in Description Logics}.}}

\author{Laura Giordano}
\institute{DISIT, University of Piemonte Orientale ``Amedeo Avogadro'' - Italy - \email{laura.giordano@uniupo.it} }

\author{Laura Giordano \inst{1} \and Valentina Gliozzi \inst{2}}

\institute{DISIT - Universit\`a del Piemonte Orientale, 
 Alessandria, Italy, {laura.giordano@uniupo.it} \and
%Center for Logic, Language and Cognition, 
Dipartimento di Informatica,
Universit\`a di Torino, Italy, {valentina.gliozzi@unito.it}
}

%\authorrunning{L. Giordano, V. Gliozzi}
%\titlerunning{Reasoning about exceptions in ontologies}

 \maketitle
 
%\pagenumbering{gobble}
 
 \begin{abstract} 
 Reasoning about exceptions in ontologies is nowadays one of the challenges the description logics community is facing. The paper describes a preferential approach for dealing with exceptions in Description Logics, based on the rational closure.  The rational closure has the merit of providing a simple and efficient approach for reasoning with exceptions, but it does not allow independent handling of the inheritance of different defeasible properties of concepts. In this work we outline a possible solution to this problem by introducing a variant of the lexicographical closure, that we call {\em skeptical closure}, 
 which requires to construct a single base. We develop a bi-preference semantics semantics for defining a characterization of the skeptical closure.

%%\keywords{Description Logics, Nonmonotonic Reasoning, Lexicographic closure}
\end{abstract}

\section{Introduction}

Reasoning about exceptions in ontologies is nowadays one of the challenges the description logics community is facing, a challenge which is at the very roots of the development of non-monotonic reasoning in the 80Õs. Many non-monotonic extensions of Description Logics (DLs) have been developed incorporating non-monotonic features from most of the non-monotonic formalisms in the literature
%The interest on the subject is demonstrated by the  number of non-monotonic extensions of description Logics(DLs), 
%which incorporate non-monotonic features from most of the non-monotonic formalisms in the literature 
\cite{baader95b,donini2002,Eiter2008,kesattler,sudafricaniKR,bonattilutz,FI09,casinistraccia2010,rosatiacm,Eiter2011,Bonatti2011,KnorrECAI12,CasiniDL2013,Gottlob14,AIJ13,AIJ15},
%cite{BaaderH92,Straccia93,baader95b,donini2002,lpar2007,Eiter2008,kesattler,sudafricaniKR,bonattilutz,casinistraccia2010,rosatiacm,Bonatti2011,CasiniDL2013,AIJ,AIJ15}, including rule-based languages \cite{Eiter2008,Eiter2011,KnorrECAI12,Gottlob14},
 or defining new constructions and semantics such as in \cite{bonattiAIJ15}.

We focus on the rational closure for DLs  \cite{casinistraccia2010,CasiniJAIR2013,CasiniDL2013,AIJ15,CasiniISWC15} and, in particular, on the construction developed in \cite{AIJ15}, which is semantically characterized by minimal (canonical) preferential models.
While the rational closure  provides a simple and efficient approach for reasoning with exceptions,
exploiting polynomial reductions to standard DLs \cite{ISMIS2015,TesiMoodley2016}, %ICTCS14
%The construction in \cite{AIJ15} requires a linear number of entailment checks in the preferential DL and,
%for DLs more expressive than $\alc$, reasoning in a preferential DL can be polynomially reduced to reasoning in the corresponding standard DL \cite{ICTCS14,ISMIS2015}. For the low complexity description logic $\sroel$ \cite{KrotzschJelia2010}, underlying the OWL EL ontology language, rational entailment can be computed in polynomial time using a Datalog calculus \cite{CILC2016}.
the rational closure does not allow an independent handling of the inheritance of different defeasible properties of concepts\footnote{%Here and in the following, 
By {\em properties} of a concept, here we generically mean characteristic features of a class of objects (represented by a set of inclusion axioms) rather than roles (properties in OWL \cite{OWL}).}
%A weaknesses of rational closure, however, is the fact that %it does not allow to reason separately property by property.
%it does not allow independent handling of the inheritance of different defeasible properties of concepts
%%so that, if the students have two typical properties of being young and of paying taxes,  then the subclasses of students which are exceptional as they are not young, do not inherit any of the properties of students. 
so that, if a subclass of $C$ is exceptional for a given aspect, it is exceptional tout court and does not inherit any of the
typical properties of $C$. 
This problem was called by Pearl \cite{PearlTARK90} ``the blocking of property inheritance problem", 
and it is an instance of the ``drowning problem" in \cite{BenferhatIJCAI93}.

To cope with this problem %of the rational closure, %a refinement of the rational closure, 
Lehmann \cite{Lehmann95} introduced the notion of the lexicographic closure, 
which was extended to Description Logics by Casini and Straccia  \cite{Casinistraccia2012}, while in  
\cite{CasiniJAIR2013} the same authors develop an inheritance-based approach for defeasible DLs.
Other proposals to deal with this ``all or nothing" behavior in the context of DLs are the the logic of overriding, ${\cal DL}^N$, 
by Bonatti, Faella, Petrova and Sauro \cite{bonattiAIJ15},  a nonmonotonic description logic in which conflicts among defaults are solved based on specificity,
and the work by Gliozzi \cite{GliozziAIIA2016}, who develops a semantics for defeasible inclusions in which  models are equipped with several preference relations.

In this paper we will consider a variant of the lexicographic closure.
The lexicographic closure allows for stronger inferences with respect to rational closure,
%As a difference with rational closure, 
but computing the defeasible consequences in the lexicographic closure  may require to compute several alternative {\em bases} \cite{Lehmann95}, namely, consistent sets of defeasible inclusions which are maximal with respect to a (so called seriousness) ordering. % to some specificity ordering). %
%In \cite{TesiMoodley2016} different kinds of closures for DLs and related algorithms have been studied,
%including algorithms for computing the  lexicographic  \cite{Casinistraccia2012} and the relevant  \cite{Casini14} closures,
%and the major bottlenecks for preferential reasoning in comparison with the rational closure are identified.
%
We propose an alternative notion of closure, the {\em skeptical closure}, which can be regarded as a more skeptical variant of the lexicographic closure.
It is a refinement of rational closure which allows for stronger inferences, but it is weaker than the lexicographic closure and its computation does not require to generate all the alternative maximally consistent bases.
%
%differently from the lexicographic closure it avoids computing alternative bases when computing 
Roughly speaking, the construction is based on the idea of building a single base, i.e. a single maximal consistent set of defeasible inclusions, 
starting with the defeasible inclusions with highest rank and progressively adding less specific inclusions, when consistent,
%including defeasible inclusions at each rank, provided they are consistent with most preferred inclusions,
but excluding the defeasible inclusions which produce a conflict at a certain stage
without considering alternative consistent bases. % (as done in the lexicographic closure).
%The approach is related to the one in \cite{bonatti2015}, but  inheritance is blocked in case of conflicting defaults with unresolved conflicts. %and, under this respect, the construction behaves differently from \cite{bonatti2015}.
Our construction only requires a polynomial number of calls to the underlying preferential $\alctr$ reasoner to be computed.
 
%It allows to reason property by property when the inherited properties are not conflicting with each other or when the conflict can be solved by overriding (where, as in the rational closure, more specific properties override more general ones).

To develop a semantic characterization of the skeptical closure, we introduce a bi-preference semantics (BP-semantics), 
which is  still in the realm of the preferential semantics for defeasible description logics
\cite{lpar2007,sudafricaniKR,FI09}, developed along the lines of the preferential semantics introduced by Kraus, Lehmann and Magidor \cite{KrausLehmannMagidor:90,whatdoes}. The BP-semantics has two preference relations and is a refinement of the rational closure semantics. 
%We exploit the BP semantics to build a semantics for the skeptical closure.
We show that the BP semantics provides a characterization of the MP-closure, a variant of the lexicographic closure introduced in \cite{Multipref_arXiv2017}. %as a sound approximation of a  multipreference semantics. 
and exploit it to build a semantics for the Skeptical closure. 
%{\color{red}[[ELIM ????showing that the Skeptical closure is weaker than the MP-closure.]]}

Schedule of the paper is the following. In Section \ref{sez:RC} we recall the definition of the rational closure for $\alc$ in \cite{AIJ15} and of its semantics. 
%In section 3, we define the new closure and in Section 4 we conclude the paper with some discussion of related work.
In Section \ref{sec:Skeptical_closure} we define the Skeptical closure. In Section  \ref{sez:BPsem}, we introduce %a refinement of the semantics of the rational closure, called 
the bi-preference semantics and, in Section \ref{sez:lex_closure} we show that it provides a semantic characterization of the MP-closure,
% a variant  of the lexicographic closure, the MP-closure, introduced in  \cite{Multipref_arXiv2017} 
a sound approximation of a  multipreference semantics in  \cite{Multipref_arXiv2017}. 
In Section \ref{sec:semantics_Skeptical_closure}, the BP-semantics is used to define a semantic characterization for the skeptical closure.
Finally, in Section  \ref{sec:conclu}, we compare with related work and conclude the paper.

This work is based on the extended abstract presented at CILC/ICTCS 2017 \cite{GiordanoICTCS2017}, where the notion of skeptical closure were first introduced.

%-------------------------------------------------

\vspace{-0.1cm}
\section{The rational closure for $\alc$}\label{sez:RC}
\vspace{-0.2 cm}
We briefly recall %the DLs $\alctr$ and $\alctr$  introduced in \cite{FI09,ecai2010DLs}, respectively. 
the logic $\alctr$  which is at the basis of a rational closure construction  proposed in \cite{AIJ15} for $\alc$.
The idea underlying $\alctr$ is that of extending the standard $\alc$ with concepts of the form $\tip(C)$,  whose intuitive meaning is that
$\tip(C)$ selects the {\em typical} instances of a concept $C$, to distinguish between the properties that
hold for all instances of concept $C$ ($C \sqsubseteq D$), and those that only hold for the typical
such instances ($\tip(C) \sqsubseteq D$).  The $\alctr$ language is defined as follows:
%that we call  \tip-inclusions, where $C$ is a concept not mentioning $\tip$.
\hide{Formally, the language is defined as follows. }
 %\vspace{-0.10cm}
 %\begin{definition}\label{defelt}
 \begin{quote}
 $C_R:= A \tc \top \tc \bot \tc  \nott C_R \tc C_R \sqcap C_R \tc C_R \sqcup C_R \tc \forall R.C_R \tc \exists R.C_R$
 
   $C_L:= C_R \tc  \tip(C_R)$, 
   
\end{quote}
   where $A$ is a concept name and $R$ a role name.
    A knowledge base $K$ is a pair $({\cal T}, {\cal A})$, where the TBox ${\cal T}$  contains a finite set
of  concept inclusions  $C_L \sqsubseteq C_R$, and the ABox ${\cal A}$
contains a finite set of assertions of the form  $C_R(a)$  and $R(a,b)$, for $a, b$ individual names, and $R$ role name.
%\end{definition}
%\vspace{-0.05cm}

The semantics of $\alctr$  is defined
in terms of rational
%\footnote{We use the expression
%``rational model'' rather than ``ranked model'' which is also used
%in the literature in order to avoid any confusion with the notion
%of rank used in rational closure.} 
models, extending to $\alc$ the preferential semantics by Kraus, Lehmann and Magidor in \cite{KrausLehmannMagidor:90,whatdoes}:
ordinary models of $\alc$ are equipped with a \emph{preference relation} $<$ on
the domain, whose intuitive meaning is to compare the ``typicality''
of domain elements: $x < y$ means that $x$ is more typical than
$y$. The instances of
$\tip(C)$ are the instances of concept $C$ that are minimal with respect to $<$.
%(i.e., there is no other member of $C$ more typical than $x$).
In rational models, which characterize $\alctr$, $<$ is further assumed to be {\em modular} (i.e., for all $x, y, z \in \Delta$, if
$x < y$ then either $x < z$ or $z < y$) and {\em well-founded}\footnote{Since $\alctr$ has the finite model property, this is equivalent to having the Smoothness Condition, as shown in \cite{AIJ15}. We choose this formulation because it is simpler.} (i.e., there is no infinite $<$-descending chain, so that, if $S \neq \emptyset$, also $min_<(S) \neq \emptyset$). 
Ranked models characterize $\alctr$. Let us shortly recap their definition.
%?????Let us shortly recap the definition from \cite{AIJ15} the definition of the rational closure of an $\alctr$ KB and of its semantic characterization in terms of minimal canonical rational models.

\vspace{-0.15cm}
\begin{definition}[Semantics of $\alctr$ \cite{AIJ15}]\label{semalctr} 
An interpretation $\emme$ of $\alctr$ is any
structure $\langle \Delta, <, I \rangle$ where: $\Delta$ is the
domain;   $<$ is an irreflexive, transitive,  modular and well-founded relation over
$\Delta$. %that satisfies the \emph{finite chain condition} 
$I$ is an interpretation function that maps each
concept name $C \in R_C$ to $C^I \subseteq \Delta$, each role name $R \in N_R$
to  $R^I \subseteq \Delta^I \times \Delta^I$
and each individual name $a \in N_I$ to $a^I \in \Delta$.
For concepts of
$\alc$, $C^I$ is defined in the usual way in $\alc$ interpetations \cite{handbook}. In particular:
\begin{quote}
   $\top^I=\Delta$
   
   $\bot^I=\vuoto$
   
   $(\neg C)^I=\Delta \backslash C^I$
   
   $(C \sqcap D)^I=C^I \cap D^I$
   
   $(C \sqcup D)^I=C^I \cup D^I$
   
   $(\forall R.C)^I=\{x \in \Delta \tc \mbox{for all } y, (x,y) \in R^I \mbox{ implies } y \in C^I\}$\
  
   $(\exists R.C)^I=\{x \in \Delta \tc \mbox{for some y } (x,y) \in R^I  \mbox{ and} \ y \in C^I\}$
\end{quote}
For the $\tip$ operator, we have
$(\tip(C))^I = min_<(C^I)$.
\end{definition}
The notion of satisfiability of a KB  in an interpretation is defined as usual.
Given an $\alc$ interpretation $I=\langle \Delta, <, I \rangle$: 

- $I$  satisfies an inclusion $C \sqsubseteq D$ if   $C^I \subseteq D^I$;

-  $I$ satisfies an assertion $C(a)$ if $a^I \in C^I$; 

- $I$ satisfies an assertion $R(a,b)$ if $(a^I,b^I) \in R^I$.

\noindent
A model $\emme$ satisfies a knowledge base $K=( {\cal T}, {\cal A} )$ if it satisfies  all the inclusions in its TBox ${\cal T}$ 
%(and for all  inclusions $C \sqsubseteq D \in {\cal T}$, it $C^I \subseteq D^I$ holds), 
and all the assertions in its ABox ${\cal A}$. %(for all $C(a) \in {\cal A}$,  $a^I \in C^I$ and, for all $R(a,b)\in {\cal A}$,  $(a^I,b^I) \in R^I$).
%\end{definition}
%
%\vspace{-0.15cm}
%
%
%\begin{definition}[Query]\label{def:query}
% A {\em query} $F$ is either an assertion $C_L(a)$ or an inclusion relation $C_L \sqsubseteq C_R$. Given a model $\emme = \langle \Delta, <, I \rangle$, a query $F=C_L(a)$ holds in $\emme$ if $a^I \in C_L^I$, whereas a query $F=C_L \sqsubseteq C_R$ holds in $\emme$ if $C_L^I \subseteq C_R^I$.
% \end{definition}
% 
% 
% \begin{definition}[Logical entailment]\label{def:entailment}
 A {\em query} $F$ (either an assertion $C_L(a)$ or an inclusion relation $C_L \sqsubseteq C_R$) is logically (rationally) entailed by a knowledge base $K$ ($K \models_{\alctr} F$) if $F$ is satisfied in all the models of $K$. %If 
 %As we consider rational models (Definition \ref{semalctr}) we will say that $F$ is rationally entailed by $K$, and we will write . 
 %\end{definition}

As shown in \cite{AIJ15}, the logic $\alctr$ enjoys the  finite model property and finite $\alctr$ models can be equivalently defined by postulating 
the existence of
a function $k_{\emme}: \Delta \longmapsto \mathbb{N}$, where $k_{\emme}$ assigns a finite rank to each world: the rank $k_{\emme}$  of a domain element $x \in \Delta$ is the
length of the longest chain $x_0 < \dots < x$ from $x$
to a minimal $x_0$ (s. t. there is no ${x'}$ with  ${x'} < x_0$). The rank $k_\emme(C_R)$ of a concept $C_R$ in $\emme$ is $i = min\{k_\emme(x):
x \in C_R^I\}$.

Although the typicality operator $\tip$ itself  is nonmonotonic (i.e.
$\tip(C) \sqsubseteq D$ does not imply $\tip(C \sqcap E)
\sqsubseteq D$), the logic $\alctr$ is monotonic: what is logically entailed by $K$ is still entailed by any $K'$ with $K \subseteq K'$.

%%%%%%%%%%% RC %%%%%

In \cite{dl2013,AIJ15} a non monotonic construction of rational closure has been defined for $\alctr$, extending
the construction of rational closure introduced by Lehmann and Magidor \cite{whatdoes} to the description logic $\alc$.  
Its definition is based on the notion of exceptionality. Roughly speaking $\tip(C) \sqsubseteq D$ holds in the rational closure of $K$ if $C$ %$C \sqcap D$ 
is less exceptional than $C \sqcap \neg D$. We shortly recall this construction of the rational closure of a TBox and we refer to \cite{AIJ15} for full details.
%Here we only consider rational closure of TBox, defined as follows.

\begin{definition} [Exceptionality of concepts and inclusions]\label{definition_exceptionality}
Let $E$ be a TBox and $C$ a concept. $C$ is {\em exceptional} for $E$ if and only if $E \models_{\alctr} \tip(\top) \sqsubseteq
\neg C$. An inclusion $\tip(C) \sqsubseteq D$ is exceptional for $E$ if $C$ is exceptional for $E$. The set of inclusions which are exceptional for $E$ will be denoted
by $\mathcal{E}$$(E)$.
\end{definition}

\noindent Given a  TBox ${\cal T}$,
it is possible to define a sequence of non increasing subsets of
the TBox ${\cal T}$ ordered according to the exceptionality of the elements $E_0 \supseteq E_1 \supseteq E_2 \dots$ by letting $E_0 ={\cal T}$ and, for
$i>0$, $E_i=\mathcal{E}$$(E_{i-1}) \unione \{ C \sqsubseteq D \in {\cal T}$ s.t. $\tip$ does not occurr in $C\}$.
Observe that, being knowledge base finite, there is
an $n\geq 0$ such that, for all $m> n, E_m = E_n$ or $E_m =\emptyset$.
%Observe also that the definition of the $E_i$'s is the same as the definition of the $C_i$'s in Lehmann and Magidor's rational closure \cite{KrausLehmannMagidor:90},
%except for that here, at each step, we also add all the ``strict'' inclusions $C \sqsubseteq D$ (where $\tip$ does not occur in $C$).
A concept $C$ has {\em rank} $i$ (denoted $\rf(C)=i$) for TBox,
iff $i$ is the least natural number for which $C$ is
not exceptional for $E_{i}$. {If $C$ is exceptional for all
$E_{i}$ then $\rf(C)=\infty$ ($C$ has no rank).}
The rank of a typicality inclusion $\tip(C) \sqsubseteq D$ is $\rf(C)$.
Observe that, for $i<j$, $E_i$ contains less specific defeasible properties then $E_j$.

\begin{example} \label{example-Student}
Let $K$ be the knowledge base with TBox:
\begin{quote}
 $\mathit{ \tip(Student) \sqsubseteq \neg Pay\_Taxes}$\\
 $\mathit{ \tip(WStudent) \sqsubseteq Pay\_Taxes}$\\
 $\mathit{\tip(Student) \sqsubseteq  Smart}$\\
 $\mathit{WStudent \sqsubseteq  Student}$
\end{quote}
stating that typical students do not pay taxes, but typical working students (which are students) do pay taxes and that  typical students are smart.
%NB se decidiamo di ampliare L_A, bisogna adattare l'esempio
It is possible to see that

$E_0 = {\cal T}$
%$\{ \mathit{ \tip(Student) \sqsubseteq}$ $\mathit{ \neg Pay\_Taxes}$,   $\mathit{\tip(Student) \sqsubseteq  Young, WStudent \sqsubseteq  Student}, \mathit{\tip( WStudent)}$ $\mathit{ \sqsubseteq Pay\_Taxes} \}$,

{ $E_1 = \{  \mathit{\tip( WStudent)}$ $\mathit{ \sqsubseteq Pay\_Taxes}$,   $\mathit{WStudent \sqsubseteq  Student} \}$}.

\noindent
In particular, the rank of concept $\mathit{Student}$ is $0$, as $\mathit{Student}$ is non-exceptional for $E_0$:
there is a model $\emme$ of the KB containing a domain element $x \in \Delta$ with rank $0$, which is an instance of $\mathit{Student}$
( $x$ satisfies all the inclusions in $E_0$).  
Instead,  $\mathit{WStudent}$ has rank $1$, as it is exceptional for $E_0$:
it is not possible to find a domain element $y$ in some model of $K$ such that
$y$ is an instance of $\mathit{WStudent}$ and has rank $0$. In fact, such a $y$ would be a typical  $\mathit{WStudent}$ and hence, it would be an instance of $\mathit{Pay\_Taxes}$ by the second inclusion. But, as a $\mathit{WStudent}$ is a Student as well, it should satisfy the first defeasible inclusion  as well and be an instance of $\mathit{ \neg Pay\_Taxes}$, which is impossible. Hence, any instance $y$ of $\mathit{WStudent}$ cannot have rank $0$.

It is easy to see that the rank of the concepts $\mathit{Student \sqcap Italian}$,  and  $\mathit{Student \sqcap }$  $\mathit{ Italian \sqcap \neg Pay\_Taxes}$  is $0$;
that the rank of concepts $\mathit{Student \sqcap Italian \sqcap Pay\_Taxes}$, $\mathit{WStudent \sqcap Italian}$ and  $\mathit{ WStudent }$ $\mathit{\sqcap Italian \sqcap Pay\_Taxes}$ is $1$; and
that the rank of concept $\mathit{ WStudent \sqcap Italian}$  $\mathit{ \sqcap \neg Pay\_Taxes}$ is $2$. 

\end{example}

Rational closure builds on this notion of exceptionality:

\begin{definition}[Rational closure of TBox] \label{def:rational closureDL}
Let $K = ({\cal T}, {\cal A})$ be a DL knowledge base. The
rational closure  of TBox is defined as: 
\begin{align*}
%    $\mbox{$\overline{\mathit{TBox}}
    RC({\cal T})= & \{\tip(C) \sqsubseteq D \in {\cal T} \tc \mbox{ either } \ \rf(C) < \rf(C \sqcap \nott D) \mbox{ or }\\
    &  \rf(C)=\infty\} \ \unione \ \{C \sqsubseteq D \in {\cal T}  \tc \ \mbox{ KB } \ \models_{\alctr} C \sqsubseteq D\}
\end{align*}
where $C$ and $D$ are $\alc$ concepts.
\end{definition}
In Example \ref{example-Student}, $\mathit{\tip(Student \sqcap Italian) \sqsubseteq  \neg Pay\_Taxes}$ is in the rational closure of the TBox,
as $\rf(\mathit{Student \sqcap Italian}) < \rf( \mathit{Student \sqcap Italian \sqcap Pay\_Taxes})$; so is 
 $\mathit{\tip(WStudent \sqcap Italian)}$ $\mathit{ \sqsubseteq   Pay\_Taxes}$.

%%%%%%%%%%%%%%%%%%%%% SEMANTICA %%%%%%%%

Exploiting the fact that entailment in $\alctr$ can be polynomially encoded into entailment in $\alc$, it is easy to see that deciding  if an inclusion $\tip(C) \sqsubseteq D$ belongs to the rational closure of TBox is a problem in \textsc{ExpTime} and requires a polynomial number of entailment checks to an $\alc$  knowledge base.
In \cite{AIJ15} it is also shown that the semantics corresponding to rational closure can be given in terms of {\em minimal canonical} $\alctr$ models. 
In such models the rank of domain elements is minimized to make each domain element to be as typical as possible.
% is as low as possible (each domain element is assumed to be as typical as possible). 
Furthermore, canonical models are considered in which all possible combinations of concepts are represented. %These two properties are expressed by the two following definitions.
This is expressed by the following definitions.

 \begin{definition}[Minimal models of $K$]\label{Preference between models in case of fixed valuation} 
Given $\emme = $$\langle \Delta, <, I \rangle$ and $\emme' =
\langle \Delta', <', I' \rangle$, we say that $\emme$ is preferred to
$\emme'$ \hide{with respect to the fixed interpretations minimal
semantics} ($\emme \prec \emme'$) if:
$\Delta = \Delta'$,
$C^I = C^{I'}$ for all (non-extended) concepts $C$,
for all $x \in \Delta$, it holds that $ k_{\emme}(x) \leq k_{\emme'}(x)$ whereas
there exists $y \in \Delta$ such that $ k_{\emme}(y) < k_{\emme'}(y)$.

Given a knowledge base $K=( {\cal T}, {\cal A} )$, we say that
$\emme$ is a minimal model of $K$ (with respect to TBox)  if it is a model satisfying $K$ and  there is no
$\emme'$ model satisfying $K$ such that $\emme' \prec \emme$.
\end{definition}
The models corresponding to rational closure are required to be canonical. This property, expressed by the following definition, is needed when reasoning about the (relative) rank of the concepts: it is important to have them all represented by some instance in the model.

\begin{definition}[Canonical model\hide{with respect to $\lingconc$}]\label{def-canonical-model-DL}
Given $K=( {\cal T}, {\cal A} )$, %and a query $F$, 
a  model $\emme=$$\sx \Delta, <, I \dx$ satisfying $K$ is %said to be
{\em canonical} %with respect to $\lingconc$} if, 
 if for each set of concepts
$\{C_1, C_2, \dots, C_n\}$% \subseteq \lingconc$ 
consistent with $K$, there exists (at least) a domain element $x \in \Delta$ such that
%$x \in C^I$ for each combination $C$ in $\mathcal{S}$ consistent with $K$.
$x \in (C_1 \sqcap C_2 \sqcap \dots \sqcap C_n)^I$. \end{definition}
%
%As said above, the intuition is that a canonical model contains all the individuals that enjoy properties that are consistent with the knowledge base. This is needed when reasoning about the (relative) rank of the concepts: it is important to have them all represented.
%
%
\begin{definition}[Minimal canonical models (with respect to TBox)]\label{def-minimal-canonical-model-DL}
$\emme$ is a minimal canonical model of $K$, 
if it is a canonical model of $K$ and it is minimal with respect $\prec$ (see Definition \ref{Preference between models in case of fixed
valuation}) among the canonical models of $K$.
%it is  minimal with respect to $TBox$ (Definition \ref{Preference between models in case of fixed valuation}) and it is canonical (Definition \ref{def-canonical-model-DL}).
\end{definition}

The correspondence between minimal canonical models and rational closure is established by the following key theorem. %STATEMENT UNICO O DA RIDURRE IN VARI STATEMENTS?
%RIPRENDI DEF RATIONAL CLOSURE

\begin{theorem}[\cite{AIJ15}]\label{Theorem_RC_TBox}
Let $K=( {\cal T}, {\cal A} )$ be a knowledge base and $C \sqsubseteq D$ a query.
Let $\overline{\mathit{TBox}}$ be the rational closure of $K$ w.r.t. TBox.
We have that $C \sqsubseteq D \in$ $\overline{\mathit{TBox}}$ if and only if $C \sqsubseteq D$ holds in all minimal  canonical models 
of $K$ with respect to TBox. 
\end{theorem}
Furthermore: the rank of a concept $C$ in any minimal canonical model of $K$ is exactly the rank $\rf(C)$ assigned by the rational closure construction,  when $\rf(C)$ is finite. Otherwise, the concept $C$ is not satisfiable in any model of the TBox. 

\begin{example}
Considering again the KB in Example \ref{example-Student}, 
we can see that defeasible inclusions  $\tip$ $\mathit{(Student \sqcap}$ $\mathit{ Italian) \sqsubseteq \neg Pay\_Taxes}$ \ and  \ 
 $\mathit{\tip(WStudent }$  $\mathit{  Italian) \sqsubseteq  Pay\_Taxes}$ are satisfied in all the minimal canonical models of $K$.
% both belong, as expected, to the rational closure of $K$.
In fact, for the first inclusion,  in all the minimal canonical models of $K$, 
$\mathit{Student \sqcap Italian}$ has rank $0$, 
while $\mathit{Student \sqcap }$  $\mathit{  Italian \sqcap Pay\_Taxes}$ has rank $1$. Thus,  in all the minimal canonical models of $K$
each typical Italian student must be an instance of $\mathit{\neg Pay\_Taxes}$.

 Instead, wethe deseafible inclusion {\small $\tip(\mathit{WStudent}) \sqsubseteq  Smart$} is not minimally entailed from $K$ and, consistently, this inclusion does not belong to the rational closure of  ${\cal T}$. 
Indeed, the concept $\mathit{WStudent}$ is exceptional for $E_0$, as it violates the defeasible property of students that, normally, they do not pay taxes  ($\mathit{ \tip(Student) \sqsubseteq \neg Pay\_Taxes}$). For this reason, $\mathit{WStudent}$ does not inherit ``any" of the defeasible properties of $\mathit{Student}$. This problem is a well known problem of rational closure, called by Pearl \cite{PearlTARK90} ``the blocking of property inheritance problem", and it is an instance of the ``drowning problem" in \cite{BenferhatIJCAI93}.

\end{example}
To overcome this weakness of the rational closure, Lehmann introduced the notion of  lexicographic closure \cite{Lehmann95}, which strengthens the rational closure by allowing, roughly speaking, a class  to inherit as many as possible of the defeasible properties of more general classes,  
giving preference to the more specific properties.
The lexical closure has been extended  to the description logic $\alc$ by Casini and Straccia in \cite{Casinistraccia2012}.
In the example above, the property of students of being smart would be inherited by working students, as it is consistent with all other (strict or defeasible) properties of  working students. 
 In the general case, however, there may be exponentially many alternative bases to be considered, which are all maximally preferred,
 % which are maximal and consistent for a given concept 
 and the lexicographic closure has to consider all of them to determine which defeasible inclusions can be accepted. %, when it holds in all the basis.
In the next section we propose an approach weaker than the lexicographic closure, 
% we consider a more skeptical approach, 
which leads to the construction of a single base.

\section{The Skeptical Closure} \label{sec:Skeptical_closure}

%The approach followed in ${\cal DL}^N$ \cite{bonattiAIJ15} inspires an alternative {\em skeptical construction} to provide a sound approximation of inferences under the MP-closure.
Given a concept $B$, one wants to identify the defeasible properties of the $B$-elements (if any). 
%including those inherited from previous ranks.
Assume that the rational closure of the knowledge base $K$ has already been constructed and that $k$ %$\rf(B)=k$ 
is the (finite) rank of concept $B$ in the rational closure\footnote{When $\rf(B) = \infty$, the defeasible inclusion $\tip(B) \sqsubseteq D$ belongs to the rational closure of TBox for any $D$. Hence, we  assume $\tip(B) \sqsubseteq D$ also belongs to the skeptical closure, and we defer considering this case until Definition  \ref{defi:skeptical_closure}. So far, we always assume $k$ to be finite.}.
The typical $B$ elements are clearly compatible (by construction) with all the defeasible inclusions in $E_k$, but they might satisfy further defeasible inclusions with lower rank,
i.e. those included in $E_0,E_1,\ldots, E_{k-1}$.

For instance, in the example above, concept $\mathit{WStudent}$ has rank $1$, and for working students all the defeasible inclusions in set $E_1$ above apply (in particular, that typical working students pay taxes). As for $E_0$,  the defeasible inclusion 
 $\mathit{ \tip(Student) \sqsubseteq \neg Pay\_}$ $\mathit{Taxes}$ is not compatible with this property of typical students, while
 the defeasible property $\mathit{\tip(Student) \sqsubseteq}$ $\mathit{  Smart}$ is, as there may be typical students which are Smart.

In general, there may be alternative maximal sets of defeasible inclusions compatible with $B$, among which one would prefer those that maximize the sets of defeasible inclusions with higher rank. %the ranks of defeasible inclusions, as well as the number of defeasible inclusions for each rank. %, with respect to those with lower rank.
This is indeed what is done by the lexicographic closure \cite{Lehmann95}, which considers alternative maximally preferred sets of defaults called "bases", which, roughly speaking, maximize the number of defaults of higher ranks with respect to those with lower ranks (degree of seriousness),
%%and considers more plausible a situation which violates $n$ defaults of a given rank w.r.t. a situation which violates less than $n$ defaults of the same rank
and where situations  which violate a number of defaults with a  certain rank are considered to be less plausible than situations which violate a lower number of defaults with the same rank.
%(and also considering the number of satisfied defaults with a given rank in comparing alternative maximal sets of defaults compatible with $B$).
%As a difference here, we do not consider the number of satisfied defaults with a given rank in comparing alternative maximal sets of defaults compatible with $B$.
In general, there may be exponentially many alternative sets of defeasible inclusions (called bases in \cite{Lehmann95}) which are maximal and consistent for a given concept, and the lexicographic closure has to consider all of them to determine if a defeasible inclusion is to be accepted or not.
As a difference, in the following, we define  a construction which skeptically builds a single set of defeasible inclusions compatible with $B$.
%As a difference, here we aim at defining a single set of defeasible inclusions which are compatible with $B$, adopting a skeptical construction.
The advantage of this construction is that it only requires a polynomial number of calls to the underlying preferential $\alctr$ reasoner.

Let $B$ a concept with rank $k$ in the rational closure.
In order to see which are the defeasible inclusions compatible with $B$ (beside those in $E_k$), we first single out the defeasible inclusions which are individually consistent with $B$ and $E_k$. 
This is done while building the set $S^B$ of the defeasible inclusions which are not overridden by those in $E_k$. 
As the set $S^B$ might not be globally consistent with $B$, for the presence of conflicting defaults,
we will consider the sets of defaults in $S^B$ with the same rank, going from $k-1$ to $0$ and we will add them to $E_k$, if consistent.
When we find an inconsistent subset, we stop.
In this way, we extend $E_k$ with all the defeasible inclusions which are not conflicting and can be inherited by $B$ instances, even though the construction of rational closure has excluded them from $E_k$.

Let $S^{B} $ be the set of typicality inclusions $\tip(C) \sqsubseteq D $ in the TBox ${\cal T}$ which are {\em individually compatible with $B$ (with respect to $E_k$)}, that is
$$S^{B} = \{\tip(C) \sqsubseteq D \in {\cal T} \mid \;  E_k \cup \{\tip(C) \sqsubseteq D\}  \not \models_{\alctr} \tip(\top) \sqsubseteq \neg B \}$$

\noindent
%Intuitively, considering that all the defeasible inclusions in $E_k$ are compatible with $B$, we consider those defeasible inclusions which can be added to $E_k$ without making $B$ impossible.????
For instance, in Example  \ref{example-Student}, for $B= \mathit{WStudent}$, which has rank $1$,  
we have that
$$S^{ \mathit{WStudent}}=\{ \mathit{\tip(Student) \sqsubseteq  Smart}, \; \mathit{ \tip(WStudent) \sqsubseteq PayTaxes}\}$$ 
is the set of defeasible inclusions compatible with $\mathit{WStudent}$ and $E_1$.
The defeasible inclusion  $\mathit{ \tip(Student) \sqsubseteq \neg Pay\_Taxes}$  %$\mathit{ \neg PayTaxes}$ has rank $0$, but it 
is not included in $S^{ \mathit{WStudent}}$
as it is not (individually) compatible with ${ \mathit{WStudent}}$.

Clearly, although each defeasible inclusion in $S^B$ is compatible with $B$, it might be the case that overall set $S^{B}$ is  not compatible with $B$, i.e., 
$$E_k \cup S^B \models_{\alctr} \tip(\top)  \sqsubseteq \neg B.$$
%where % for any set of defeasible inclusions $S$,  
%$\tilde{S}$ is the materialization of $S$, i.e.,  $\tilde{S}= {\sqcap} \{ (\neg C \sqcup D) \mid  \tip(C) \sqsubseteq D \in S \}$.

%
Let us consider the following variant of Example \ref{example-Student}.
\begin{example} \label{example-Student-new}
Let $K'$ be the knowledge base with the TBox:
\begin{quote}
 $\mathit{ \tip(Student) \sqsubseteq  Young}$\\
 $\mathit{ \tip(Student) \sqsubseteq \neg PayTaxes}$\\
 $\mathit{ \tip(Employee) \sqsubseteq PayTaxes}$\\
 $\mathit{\tip(Student \sqcap Employee) \sqsubseteq  \neg Young}$
%$\mathit{Student  \sqsubseteq  Person}$\\
%$\mathit{Employee  \sqsubseteq  Person}$\\
%$\mathit{Person  \sqsubseteq  \exists hasSSN.\top}$
\end{quote}
Let $\mathit{B=Student  \sqcap  Employee}$.
While concepts  $\mathit{Student}$ and $\mathit{Employee}$ have rank $0$, concept 
$\mathit{Student}$ $\mathit{  \sqcap  Employee}$ has rank $1$. 
In this example:
\begin{quote}
$E_0 = {\cal T}$ 
%\{ \mathit{ \tip(Student) \sqsubseteq}$ $\mathit{ \neg Pay\_Taxes, \; \tip(Student) \sqsubseteq  Young}$,   

%$\;$ \mbox{\ \ \ \ \ \ \ \ \ \ }$\mathit{  \tip(Employee) \sqsubseteq PayTaxes  \}}$,

$E_1 =\mathit{Strict_{\cal T}} \cup  \{  \mathit{\tip( Student \sqcap Employee) \sqsubseteq \neg Young} \}$
\end{quote}
where $\mathit{Strict_{\cal T}} $ is the set of strict inclusions in ${\cal T}$.
The property that typical employed students are not young, overrides the property that students are typically young. Indeed the default $\mathit{ \tip(Student) \sqsubseteq  Young}$ is not individually compatible with $\mathit{Student  \sqcap  Employee}$.
Instead, the defeasible properties  $\mathit{ \tip(Student) \sqsubseteq}$ $\mathit{ \neg PayTaxes}$ and $\mathit{ \tip(Employee) \sqsubseteq  Pay}$- $\mathit{ Taxes}$
are both individually compatible with   $\mathit{Student  \sqcap}$ $\mathit{Employee}$, and  
$$S^B= \{\mathit{ \tip(Student) \sqsubseteq \neg PayTaxes}, \; \mathit{ \tip(Employee) \sqsubseteq PayTaxes}\}.$$
Nevertheless, the overall set $S^{B}$ is  not compatible with $\mathit{Student  \sqcap}$ $\mathit{Employee}$. In fact, the two defeasible inclusions in  $S^B$ are conflicting.
\end{example}

When compatible with $B$,  $S^{B}$ is %the unique (subset inclusion) maximal set of defeasible inclusions compatible with $B$, and also 
the unique maximal basis with respect to the {\em seriousness ordering} in \cite{Lehmann95} (as defined for constructing the lexicographic closure).
%\begin{proposition}
%If $ E_k \not \models_{\alctr} \tip(\top) \sqcap \tilde{S^B}  \sqsubseteq \neg B$, $S^{B}$ is the unique (subset inclusion)  maximal set of defeasible inclusions compatible with $B$.
%\end{proposition}
%

%Let $ \delta(E_{i})$ denote the set of defeasible inclusions in $E_i$. 
When $S^B$ is not compatible with $B$, we cannot use all the defeasible inclusions in $S^B$ to derive conclusions about typical $B$ elements.
In this case, we can either just use the defeasible inclusions in $E_k$, as in the rational closure, or we can additionally use a subset of the defeasible inclusions $S^B$.
This is essentially what is done in the lexicoghaphic closure, where (in essence) the most preferred subsets of $S^B$ are selected according to a lexicographic order, which prefers defaults with higher ranks to defaults with lower ranks.
%In our construction instead, we follow the approach of considering,  the subsets $S_{0}^B, S_{1}^B \ldots S_{k-1}^B$ of the set $S^B$ defined above,
%where $S_{i}^B= S^B \cup \delta (E_i)$ is the set of the defeasible inclusions with rank $i$, which are  (individually) compatible with $B$.
%Starting from the highest rank $k-1$, we add  all the defeasible inclusions in $S_{k-1}^B$ to $E_k$  if  $S_{k-1}^B$  (globally) is compatible with $B$ and $E_k$.
%Then we add all the defeasible inclusions in $S_{k-2}^B$  to $E_k \cup S_{k-1}^B$, if $S_{k-2}^B$ is (globally) compatible with $B$ and $E_k \cup S_{k-1}^B$,
%and so on and so forth. We stop when we find a set $S_{i}^B$ of defeasible inclusions which is not (globally) compatible with $B$ and $E_k \cup S_{k-1}^B, \ldots, S_{i+1}^B$.
In our construction instead, we consider the subsets $S_{0}^B, S_{1}^B \ldots S_{k-1}^B$ of the set $S^B$ defined above,
by adding to $E_k$  
all the defeasible inclusions in $S^B$ with rank $k-1$ %(i.e., the most specific defeasible inclusions in $S^B$; 
(let us call this set $S_{k-1}^B$),
%which are (individually) compatible with $B$ and $E_k$, 
provided they are (altogether) compatible with $B$ and $E_k$.
Then, we can add all the defeasible inclusions with rank $k-2$ which are individually compatible with $B$ w.r.t. $E_k \cup S_{k-1}^B$ (let us call this set $S_{k-2}^B$),
%which are (individually) compatible with $B$, $E_k$ and $S_{k-1}^B$,
provided they are altogether compatible with $B$, $E_k$ and $S_{k-1}^B$, and so on and so forth,
for lower ranks. This leads to the construction below.

\begin{definition} \label{def:global_compatibility}
Given two sets of defeasible inclusions $S$ and $S'$,
 {\em $S$ is globally compatible with $B$ w.r.t.  $E_k \cup S'$} if
$$ E_k  \cup S \cup  S'  \not \models_{\alctr} \tip(\top)   \sqsubseteq \neg B$$
\end{definition}

\begin{definition}
Let $B$ be a concept such that  $\rf(B)=k$ ($k$ finite).
The {\em skeptical closure of $K$ with respect to $B$}  is the set of inclusions %as follows: %skeptical set of defeasible inclusions for $B$
$S^{sk,B} = E_k \cup S_{k-1}^B \cup {S}_{k-2}^B \cup  \ldots \cup  S_{h}^B$
where:
\begin{itemize}
\item
$S_i^B \subseteq E_{i} - E_{i+1}$ is the set of defeasible inclusions with rank $i$ which are individually compatible with $B$ w.r.t. $E_k \cup S_{k-1}^B \cup {S}_{k-2}^B \cup  \ldots \cup  S_{i+1}^B$ (for each finite rank $i < k$);
\item 
$h$ is the least $j$ (for $0\leq j < k$) such that  $S_{j}^B$ is globally compatible with $B$ w.r.t.  $E_k \cup S_{k-1}^B \cup {S}_{k-2}^B \cup  \ldots \cup  S_{j+1}^B$, 
%i.e.,  $ E_k \cup S_{k-1}^B \cup {S}_{k-2}^B \cup  \ldots \cup  S_{j}^B  \not \models_{\alctr} \tip(\top)  \sqsubseteq \neg B$, 
if  such a $j$ exists; $S^{sk,B}= E_k$, otherwise. 
\end{itemize}
\end{definition}
Intuitively, $S^{sk,B}$ contains, for each rank $j$, all the defeasible inclusions having rank $j$ which are compatible with $B$ and with the more specific defeasible inclusions (having rank $>j$). As $S_{h-1}^B$ is not included in the skeptical closure, it must be that
$ E_k  \cup S_{k-1}^B \cup {S}_{k-2}^B \cup  \ldots \cup  S_h \cup  S_{h-1}^B \models_{\alctr} \tip(\top) \sqsubseteq \neg B$
i.e., the set $S_{h-1}^B$ contains conflicting defeasible inclusions which are not overridden by more specific ones. 
In this case, the inclusions in $S_{h-1}^B$ (and, similarly, all the defeasible inclusions with rank lower than $h-1$) are not included in the skeptical closure w.r.t. $B$. %Similarly, for all $j<h-1$, the inclusions in $S_{j}^B$ are not added to the skeptical closure of $B$.% In fact, by the monotonicity of entailment, $ E_k  \cup S_{k-1}^B \cup S_{k-2}^B \cup \ldots \cup  S_{h-1}^B \cup  \ldots \cup  S_{j}^B   \models_{\alctr} \tip(\top) \sqsubseteq \neg B$.

\begin{example} \label{es:Sk_es1}
For the knowledge base $K$ in Example \ref{example-Student}, where $B= \mathit{WStudent}$ has rank $1$, 
we have $S_0^{ B}=\{ \mathit{\tip(Student) \sqsubseteq  Smart}\}$, which is compatible with $\mathit{WStudent}$ and $E_1$.
Hence, $S^{sk,B} = E_1 \cup S_{0}^B$.
\end{example}

When a defeasible inclusion belongs to the skeptical closure of a TBox is defined as follows.
\begin{definition} \label{defi:skeptical_closure}
Let $K= ({\cal T}, {\cal A})$ be a knowledge base and
 $\tip(B) \sqsubseteq D$  a query. %such that $k=\rf(B)$ is the rank of concept $B$ in the rational closure of ${\cal T}$.
$\tip(B) \sqsubseteq D$ {\em is in the skeptical closure of ${\cal T}$} if  either $\rf(B)= \infty$ in the rational closure of ${\cal T}$ or
$ S^{sk,B}  \models_{\alctr} \tip(\top) \sqsubseteq (\neg B \sqcup D)$.
%where $S^{sk,B}$ is the skeptical closure of $B$.
\end{definition}

\noindent Once the rational closure of TBox has been computed, 
the identification of the defeasible inclusions in %the skeptical closure of $B$, %
$S^{sk,B}$ requires a  number of entailment checks which is linear in the number of defeasible inclusions in TBox.
First, the compatibility of each defeasible inclusion in TBox with $B$ has to be checked to compute all the $S_j^B$'s.
Then,  a compatibility check for each rank of the rational closure is needed, to verify the compatibility of $S_j^B$, for each $j$ from $k-1$ to $0$ in the worst case. The maximum number or ranks in the rational closure is bounded by the number of defeasible inclusions in TBox (but it might be significantly lower in practical cases). 
Hence, computing the skeptical closure for $B$ requires a  number of entailment checks which is, in the worst case, $O(2 \times|{\cal T}|)$.

\begin{example} 
For the knowledge base $K$ in Example  \ref{example-Student}, we have seen that, for $B= \mathit{WStudent}$ (with rank $1$), 
 $S_0^{ B}=\{ \mathit{\tip(Student) \sqsubseteq  Smart}\}$ is (globally) compatible with $\mathit{WStudent}$ w.r.t. $E_1$, and $S^{sk,B} = E_1 \cup S_{0}^B$.
It is easy to see that $ S^{sk,B}  \models_{\alctr} \tip(\top) \sqsubseteq \mathit{(\neg WStudent \sqcup Smart})$, and that
 $\mathit{\tip(WStudent) \sqsubseteq Smart}$ is in the skeptical closure of TBox.
 In this case, the typical property of students of being Smart is inherited by working students.

%In Example  \ref{example-Student} the inclusion $\mathit{\tip(WStudent) \sqsubseteq Young}$ is in the skeptical closure of TBox, as $\mathit{WStudent}$ has rank $1$ and  inclusion $\mathit{\tip(Student) \sqsubseteq Young}$ in $E_0$ is compatible with $\mathit{WStudent}$. %(w.r.t. $E_1$). 
%No other inclusions in $\delta(E_0)$ are compatible with $E_1$.
\end{example}

\begin{example} 
For the knowledge base $K'$ in Example \ref{example-Student-new}, as we have seen,  $B= \mathit{Student \sqcap}$  $\mathit{Employee}$ has rank $1$, 
$E_1= \{ \mathit{\tip(Student \sqcap Employee) \sqsubseteq \neg }$ $\mathit{ Young}\}$, and 
$S^B =\{ \mathit{ \tip(Student)  \sqsubseteq \neg PayTaxes,}$ $\mathit{ \tip(Employee) \sqsubseteq PayTaxes}\}$.
In this case, as $S_0^B = S^B $ contains conflicting defaults about tax payment,
$S_0^B$ is not (globally) compatible with $ \mathit{Student \sqcap}$ $\mathit{  Employee}$ and $E_1$, so that $S^{sk,B} = E_1$.
\end{example}

Let us consider the following knowledge base  from \cite{Multipref_arXiv2017} to see that, in the skeptical closure, inheritance of defeasible properties, when not overridden for more specific concepts, applies to concepts of all ranks.

\begin{example} \label{exa:BabyPenguin} 
Consider a knowledge base $K=({\cal T} ,{\cal A})$, where ${\cal A}=\emptyset$ and  ${\cal T}$ contains the following inclusions: 

$\mathit{\tip(Bird) \sqsubseteq  Fly}$

$\mathit{\tip(Bird) \sqsubseteq  NiceFeather}$

$\mathit{Penguin \sqsubseteq Bird}$

$\mathit{\tip(Penguin) \sqsubseteq \neg Fly}$

$\mathit{\tip(Penguin) \sqsubseteq BlackFeather}$ 

$\mathit{BabyPenguin \sqsubseteq Penguin}$

$\mathit{\tip(BabyPenguin)}$ $ \sqsubseteq$ $\mathit{ \neg BlackFeather}$.

\noindent
Here, we expect that the defeasible property of birds having a nice feather is inherited by typical penguins,  even though penguins are exceptional birds regarding flying.
We also expect that typical baby penguins inherit the defeasible property of penguins that they do not fly, %(by presumption of independence \cite{Lehmann95}), 
although the defeasible property $\mathit{BlackFeather}$ is instead overridden for typical baby penguins, and that they inherit the typical property of birds of  having nice feather.
We have that $\rf(\mathit{Bird})=0$, $\rf(\mathit{Penguin})=1$, $\rf(\mathit{BabyPenguin})=2$ as, in the rational closure construction:
\begin{quote}
$E_0 =  \mathit{Strict_{\cal T}} \cup \{ \mathit{\tip(Bird) \sqsubseteq  Fly}$, \  \ $\mathit{\tip(Bird) \sqsubseteq  NiceFeather}   \}$

$E_1 = \mathit{Strict_{\cal T}} \cup \{ \mathit{\tip(Penguin) \sqsubseteq \neg Fly}$, \ \ $\mathit{\tip(Penguin) \sqsubseteq BlackFeather}  \}$

$E_2 = \mathit{Strict_{\cal T}} \cup \{  \mathit{\tip(BabyPenguin)}$ $ \sqsubseteq$ $\mathit{ \neg BlackFeather}\}$
\end{quote}
In particular, for $B= \mathit{BabyPenguin}$, we get
\begin{quote}
 $S_1^B =  \{ \mathit{\tip(Penguin) \sqsubseteq \neg Fly}  \}$
 
  $S_0^B =  \{ \mathit{\tip(Bird) \sqsubseteq  NiceFeather}  \}$
\end{quote}
Also, $S_1^B$ is (globally) consistent with $E_2$, and $S_0^B$ is (globally) consistent with $E_2 \cup S_1^B$. Hence,
$S^{sk,B} =E_2 \cup S_1^B \cup S_0^B =  \{  \mathit{\tip(BabyPenguin)}$ $ \sqsubseteq$ $\mathit{ \neg BlackFeather}, \;  \mathit{\tip(Penguin)}$ $\mathit{ \sqsubseteq \neg Fly} $, \ \ $\mathit{\tip(Bird) \sqsubseteq }$ $\mathit{  NiceFeather} \}$. 
Furthermore, $$\mathit{\tip(BabyPenguin) \sqsubseteq NiceFeather \sqcap \neg Fly \sqcap  \neg BlackFeather}$$ is in the skeptical closure of TBox ${\cal T}$ as
$ S^{sk,B}  \models_{\alctr} \tip(\top) \sqsubseteq (\neg \mathit{BabyPenguin}$  $ \sqcup \mathit{  (NiceFeather \sqcap}$ $\mathit{ \neg Fly \sqcap  \neg BlackFeather)}$.

\end{example}

To see that the notion of skeptical closure is rather weak, let us slightly modify the KB in Example  \ref{example-Student-new} (removing the last inclusion).
\begin{example}  \label{exa:weak}
Consider the TBox % (in which we have removed the last inclusion in $K'$):
\begin{quote}
 $\mathit{ \tip(Student) \sqsubseteq  Young}$\\
 $\mathit{ \tip(Student) \sqsubseteq \neg PayTaxes}$\\
 $\mathit{ \tip(Employee) \sqsubseteq PayTaxes}$
 %$\mathit{\tip(Student \sqcap Employee) \sqsubseteq  \neg Young}$
%$\mathit{Student  \sqsubseteq  Person}$\\
%$\mathit{Employee  \sqsubseteq  Person}$\\
%$\mathit{Person  \sqsubseteq  \exists hasSSN.\top}$
\end{quote}
As in Example  \ref{example-Student-new}, the rational closure assigns rank $0$ to concepts $\mathit{Student}$ and $\mathit{Employee}$ and
rank $1$ to $\mathit{Student  \sqcap  Employee}$.
In this case, 
%\begin{quote}
%$E_0 = \{ \mathit{ \tip(Student) \sqsubseteq}$ $\mathit{ \neg Pay\_Taxes, \; \tip(Student) \sqsubseteq  Young}$, \\
%$\mathit{  \tip(Employee) \sqsubseteq PayTaxes  \}}$;
%
%$E_1 = \emptyset$;
%
% $S_0^B =  \{ \mathit{ \tip(Student) \sqsubseteq}$ $\mathit{ \neg Pay\_Taxes, \; \tip(Student) \sqsubseteq  Young}$, \ $\mathit{  \tip(Employee) \sqsubseteq PayTaxes  \}}$.
%\end{quote}
\begin{quote}
$E_0 = \{ \mathit{ \tip(Student) \sqsubseteq}$ $\mathit{ \neg PayTaxes, \; \tip(Student) \sqsubseteq  Young}$,   

$\;$ \mbox{\ \ \ \ \ \ \ \ \ \ }$\mathit{  \tip(Employee) \sqsubseteq PayTaxes  \}}$;

$E_1 = \emptyset$;\\

 $S_0^B =  \{ \mathit{ \tip(Student) \sqsubseteq}$ $\mathit{ \neg PayTaxes, \; \tip(Student) \sqsubseteq  Young}$. 

$\;$ \mbox{\ \ \ \ \ \ \ \ \ \ \ \ \ \ \ \ \ \ \ \ \ }$\mathit{  \tip(Employee) \sqsubseteq PayTaxes  \}}$.
\end{quote}
As $S_0^B$ is not (globally) compatible with $ \mathit{Student \sqcap}$ $\mathit{  Employee}$ and $E_1$, again $S^{sk,B} = E_1$.
Therefore, the defeasible property that typical students are young is not inherited by typical employed students.
\end{example}
%This last example shows that t
The skeptical closure is a weak construction:
in Example \ref{exa:weak} due to the conflicting defaults concerning tax payment for $\mathit{Employee}$ and $\mathit{Student}$ (both with rank 0)
also the property that typical students are young %(a defeasible inclusion of rank $0$) 
is not inherited by the typical employed students.
 % {\small $\tip(Student) \sqsubseteq \neg Pay\_Taxes$} is not inherited by $WStudent$ as well.
% However, it neither inherits the property that typical students are young.
Notice that, the property that typical working students are young would be accepted in the lexicographic closure of $K'$,  % {\small $\tip(WorkingStudent) \sqsubseteq Young$} from $K'$, as it holds in all the (two)  bases of the KB. 
 as there are two bases, the one including { $\mathit{\tip(Student) \sqsubseteq}$ $\mathit{ \neg PayTaxes}$} and the other including { $ \mathit{\tip(Employee) \sqsubseteq PayTaxes}$}, both containing { $\mathit{\tip(Student)}$ $\mathit{ \sqsubseteq  Young}$}.
 The skeptical closure is indeed weaker than the lexicographic closure (and, in particular, $\mathit{\tip(Student \sqcap Employee)}$ $\mathit{ \sqsubseteq  Young}$ would be  in the lexicographic closure as defined in \cite{Casinistraccia2012}).

%In the next section we will study the semantics of the skeptical closure and we show that it can be considered as a weaker 
%but sound approximation of the lexicographic closure.????

  In the next section, we introduce a semantics based on two preference relations.
% the one playing the role of the preference relation on the preferential semantics, and the other providing a refinement of the first. 
We will show that this semantics characterizes a variant of the lexicographic closure introduced in \cite{Multipref_arXiv2017} and exploit it to define a semantic construction for the weaker skeptical closure.

 \section{Refined, bi-preference Interpretations} \label{sez:BPsem}

To capture the semantics of the skeptical closure, we build on the preferential semantics for rational closure of $\alctr$,
introducing a notion of {\em refined, bi-preference interpretation} (for short, BP-interpretation), 
which contains an additional notion of preference with respect to 
an $\alctr$ interpretation. We let an interpretation to be a tuple $\emme = \langle \Delta, <_{rc}, <, I \rangle$, where the triple $\langle \Delta, <_{rc}, I \rangle$ is a ranked interpretation as defined in Section \ref{sez:RC}.
 and $<$ is an additional preference relation over $\Delta$,  with the properties of being irreflexive, transitive and well-founded (but we do not require modularity of $<$).
 In BP-interpretations, $<$ represents a refinement of $<_{rc}$.

 \begin{definition}[BP-interpretation]\label{BP-interpretation}
Given a knowledge base K, a {\em  bi-preference interpretation} (or BP-interpretation) is  a structure 
$\emme = \langle \Delta, <_{rc}, <, I \rangle$, where
$\Delta$ is a domain, $I$ is an interpretation function as defined in Definition \ref{semalctr}, where, in particular, $(\tip(C))^I = min_<(C^I)$,
and $<_{rc}$ and $<$ are preference relations over $\Delta$, with the properties of being irreflexive, transitive, well-founded. Furthermore $<_{rc}$ is modular. 
\end{definition}

The bi-preference semantics, builds on a ranked semantics for the preference relation $<_{rc}$,
providing a characterization of the rational closure of $K$, and exploits it to define the preference relation $<$ which is not required to be modular.
%Observe that, in essence, condition $(a)$ uses the {\em pareto ordering} of the relations $<_{A_i}$, which is not a modular relation, although each $<_{A_i}$ is a modular relation.
As we will see, this semantics provides a sound and complete characterization of a variant of the lexicographic closure, and we will use it as well
to provide a semantic characterization of the skeptical closure.
The BP-semantics has some relation with the multipreference semantics in \cite{Multipref_arXiv2017}. However, it does non exploits multiple preferences w.r.t. aspects and it directly build on the preference relation $<_{rc}$. Also, in BP-interpretations, $<$ is not required to be modular.

Let $k_{\emme,rc}$ be the ranking function associated in $\emme$ to the modular relation $<_{rc}$, which is defined as the ranking function $k_{\emme}$ %associated with the modular preference relation $<$  
in the models of the rational closure in Section \ref{sez:RC}.
Similarly, the ranking function is extended to concepts by defining
the rank $k_{\emme,rc}(C)$ of a concept $C$ in a BP-interpretation $\emme$ (w.r.t. the preference relation  $<_{rc}$) as $k_{\emme,rc}(C) = min\{k_{\emme,rc}(x): x \in C^I\}$.

Given a BP-interpretation $\emme = \langle \Delta, <_{rc}, <, I \rangle$ and an element $x \in \Delta$, we say that {\em  $x$ violates the typicality inclusion $\tip(C) \sqsubseteq D$} if $x \in (C \sqcap \neg D)^I$.
Let us define when a BP-interpretation is a model of a knowledge base $K$:

\begin{definition}[BP-model of K]\label{multipreference-model-of-K}
Given a knowledge base K,  a BP-interpretation $\emme = \langle \Delta, <_{rc} <, I \rangle$
is a {\em BP-model of $K$} if it satisfies both its TBox and its ABox, in the following sense:
%$\emme$  satisfies TBox if 
\begin{itemize}
\item[(1)]
for all strict inclusions $C \sqsubseteq D$ in the TBox
(i.e., $\tip$ does not occur in $C$), $C^I \subseteq {D}^I$;
\item[(2)]
for all typicality inclusions $\tip(C) \sqsubseteq D$ in the TBox,   $min_{<_{rc}}({C}^I) \subseteq {D}^I$;

\item[(3)]
$<$ satisfies the following specificity condition:\\
%($a_1$)\  $x < y$ if there is $A_i$ such that $x <_{A_i} y$, and there is no $A_j$ such that $y <_{A_j} x$, or \\
{\bf \em If} 
\begin{itemize}
\item[-]  there is some $\tip(C) \sqsubseteq D \in K$ which is violated by $y$ and,

%$ y \in (C_i \sqcap \neg A_i)^I$ and, \\
\item[-] for all $\tip(C_j) \sqsubseteq D_j \in K$, which is violated by $x$ and not by $y$,
 %$x \in (C_j \sqcap \neg A_j)^I$  and $y \not \in (C_j \sqcap \neg A_j)^I$, \\
 there is a $\tip(C_k) \sqsubseteq D_k \in K$,  which is violated by $y$ and not by $x$, and such that $k_{\emme,rc}(C_j) < k_{\emme,rc}(C_k)$,
 %s.t.  $y \in (C_k \sqcap \neg A_k)^I$,  $x \not \in (C_k \sqcap \neg A_k)^I$ and \\
% $\;$ \mbox{ \ \ \ \ \ \ \ \ \ \ \ \ \ \ \ \ \ \ \ \ \ \ \ \ \ \ \ \ \ \ \ \ \ \ \ \ \ \ \ \ \ \ \ \ \ \ \ \ } $k_{\emme,rc}(C_j) < k_{\emme,rc}(C_k)$, \\

\end{itemize}
{\bf \em then} $x < y$ ;
%(where $k_{\emme,rc}$ is the ranking function associated in $\emme$ with the modular relation $<_{rc}$, as described in Section \ref{sez:RC});
\item[(4)] for all $C(a)$  in ABox,  $a^I \in C^I$; and, for all $R(a,b)$ in ABox,  $(a^I,b^I) \in R^I$;
\end{itemize}
%$\emme$ satisfies ABox  if:  (i) for all $C(a)$  in ABox,  $a^I \in C^I$, (ii) for all $R(a,b)$ in ABox,  $(a^I,b^I) \in R^I$.
\end{definition}
While the satisfiability conditions (1), (2) and (4) are the same as in Section \ref{sez:RC} for the ranked model $ \langle \Delta, <_{rc}, I \rangle$,
the specificity condition (3) requires the relation $<$ to satisfy the condition that,  if $y$ violates defeasible inclusions more specific than those violated by $x$, then $x<y$
% if two domain elements $x$ and $y$ violate different defeasible inclusions, the violation of a less specific defeasible inclusions (e.g., the one for $C_j$) is preferred to the violation of a more specific one (e.g., the one for $C_k$).
(in particular, the condition  $k_{\emme,rc}(C_j) < k_{\emme,rc}(C_k)$ means that concept $C_k$ is more specific than concept $C_j$, as it has an higher rank in the rational closure).

In the definition above we do not impose the further requirement that, for all inclusions $\tip(C) \sqsubseteq {D}$, $min_{<}({C}^I) \subseteq {D}^I$ holds. % as in Definition  \ref{enriched-model-of-K}. 
However, can easily see that this condition follows from condition (2) and from the property that $<_{rc} \subseteq <$ holds.
%In fact, $<_{rc} \subseteq <$. If $x <_{rc}y$, for some $j$, $k_{\emme,rc}(x)=j < k_{\emme,rc}(y)$. There must be some defeasible inclusion of rank $j$ falsified by $y$, while all the defeasible inclusions with rank up to $j$ are satisfied by $x$. Hence, by (3) $x<y$. 

\begin{proposition} \label{prop:<}
Given a knowledge base $K$ and a BP-model $\emme = \langle \Delta,  <_{rc}, <, I  \rangle$ of $K$,
%for all inclusions $\tip(C)^I \sqsubseteq {D}^I$, $min_{<}({C}^I) \subseteq {D}^I$ holds.
$<_{rc} \subseteq <$.
\end{proposition}
\begin{proof}
%From item (2) in Definition \ref{multipreference-model-of-K}, we know that $min_{<_{rc}}({C'}^I) \subseteq {A_i}^I$.
%We prove that $<_{rc} \subseteq <$, from which it follows that $min_{<}({C}^I) \subseteq min_{<_{rc}}({C'}^I)$.
We show that $x <_{rc}y$ implies If $x <y$.
If $x <_{rc}y$, then for some $r$, $k_{\emme,rc}(x)=r < k_{\emme,rc}(y)$. 
As $\emme$ is a minimal canonical BP-model of $K$, by the correspondence with the rational closure, 
$x$ satisfies all the defeasible inclusions in $E_r$. Instead, $y$ falsifies some defeasible inclusion $\tip(C_k) \sqsubseteq D_k$ with $\rf(C_k)= r$.
As $x$ can only falsify defeasible inclusions with rank less then $r$, by condition (3) in Definition \ref{multipreference-model-of-K}, $x<y$. 
Therefore,  $<_{rc} \subseteq <$.
%$\hfill \bbox$
\end{proof}

\begin{corollary} \label{corollary}
Given a knowledge base $K$ and a BP-model $\emme = \langle \Delta,  <_{rc}, <, I  \rangle$ of $K$,
for all inclusions $\tip(C)^I \sqsubseteq {D}^I$, $min_{<}({C}^I) \subseteq {D}^I$ holds.
\end{corollary}
\begin{proof}
From item (2) in Definition \ref{multipreference-model-of-K}, we know that $min_{<_{rc}}({C}^I) \subseteq D^I$.
By Proposition \ref{prop:<}, $<_{rc} \subseteq <$, from which it follows that $min_{<}({C}^I) \subseteq min_{<_{rc}}({C'}^I)$.
Hence, the thesis follows.
\end{proof}

We define logical entailment under the BP-semantics as follows:
a query $F$ (of the form $C_L(a)$ or $C_L \sqsubseteq C_R$) is logically entailed by $K$ in $\alctrbp$ (written $K \models_{\alctrbp} F$) if $F$ holds in all BP-models of $K$.

The following result can be easily proved for BP-entailment: 
 \begin{theorem}\label{equivalence-multiple3}
If $K \models_{\alctr} F$ then also $K \models_{\alctrbp} F$. If $\tip$ does not occur in $F$ the other direction also holds: If $K \models_{\alctrbp} F$ then also $K \models_{\alctr} F$. 
 \end{theorem} 
%\color{blue} {AGGIUNGERE PROVA} \normalcolor

To define a notion of minimal canonical BP-model for $K$, we proceed as in the semantic characterization of the rational closure in Section   \ref{sez:RC}.
Let the a function $d_\emme$ associated with the preference relation $<$ be such that, for any element $x \in \Delta$:
if $x \in min_<(\Delta)$, then $d_\emme(x)=0$; otherwise,
 $d_\emme(x)$ is the length of the longest path $x_0 <x_1< \ldots < x$ from $x$ to an element $x_0$ such that $d_\emme(x)=0$.
 
 %
%Observe that a non modular preference relation $<$, cannot be equivalently defined by postulating the existence of a function $k_{\emme}: \Delta \longmapsto \mathbb{N}$ (as for modular preference relations). Nevertheless, 
Although $<$ is not assumed to be modular, for each domain element $x$,  $d_\emme(x)$ represents the distance of $x$ from the most preferred elements in the model, and can be used for defining the a notion of preference  $\prec_{BP}$ among BP-models of $K$.
Let $Min_{RC}(K)$ be the set of all BP-models $\emme = \langle \Delta, <_{rc} <, I \rangle$ of $K$ such that
$ \langle \Delta, <_{rc}, I \rangle$ is a minimal canonical model of $K$ according to the semantics of rational closure in Section \ref{sez:RC} (Definition \ref{def-minimal-canonical-model-DL}).
Thus, the models in $Min_{RC}(K)$ are those built from the minimal canonical models of the rational closure of $K$.
The minimal (canonical) BP-models of $K$ will be the models in $Min_{RC}(K)$ which also minimize the distance $d_\emme(x)$  of each domain element $x$.

%$Min_{RC}(K)= \{ \emme \;  \mid \;$ $\emme = \langle \Delta, <_{rc} <, I \rangle$ is a BP-model of $K$. and 

\begin{definition}[Minimal canonical BP-Models]\label{minimal_multipreference_models} 
Given two BP-models of $K$, $\emme = \langle \Delta,  <_{rc}, <, I  \rangle$ and $\emme' =
\langle \Delta',   <'_{rc},  <', I' \rangle$ in $Min_{RC}(K)$, \ 
 {\em $\emme'$ is preferred to $\emme$} (written $\emme' \prec_{BP} \emme$) if 
 \begin{itemize}
 \item $\Delta = \Delta'$, $I = I'$, and
\item   for all $x \in \Delta$, $ d_{{\emme'}}(x) \leq d_{{\emme}}(x)$;
\item for some $y \in \Delta$ ,  $ d_{{\emme'}}(y) < d_{{\emme}}(y)$ 
\end{itemize}

A BP-interpretation $\emme$ is a {\em minimal canonical BP-model} of $K$ if $\emme$ is a model of $K$, $\emme \in Min_{RC}(K)$ and there is no $\emme' \in Min_{RC}(K)$  such that  {\small $ \emme' \prec_{BP} \emme$}. 
\end{definition}
We denote by $\models_{BP}^{min}$ entailment with respect to minimal canonical BP-models: for a query $F$,
 $K \models_{BP}^{min} F$ if $F$ is satisfied in all the minimal canonical BP-models of $K$.

Observe that, according to this definition, for computing the minimal (canonical) BP-models of $K$
one first needs to compute the set of the minimal (canonical) models of$K$ which characterize rational closure of $K$.
%so that $k_{\emme,rc}(C)$ will correspond to the rank of a concept in the rational closure.
Then, among such models, one has to select those which are minimal  with respect to $\prec_{BP}$.

Clearly, as minimal canonical BP-models of a KB are minimal ranked models as defined in Section \ref{sez:RC}, %which characterize rational closure,
$\prec_{rc}$ correspond to the preference relation in minimal canonical models of the rational closure, and
 the rank $k_{{\emme},rc}(x)$ of domain elements will be the same as in the minimal models of rational closure.
 Thus, by Theorem  \ref{Theorem_RC_TBox}, the value of $k_{{\emme},rc}(C)$, for any concept $C$, in a minimal canonical BP-model is equal to $\rf(C)$,  the rank assigned to $C$ by the rational closure construction in Section  \ref{sez:RC}. 

The rank of domain elements with respect to $<_{rc}$
is used to determine the preference relation $<$ on domain elements, according to condition (3).
Minimization with respect to $<$ is needed to guarantee that $<$ is minimal, among all the reflexive and transitive preference relations $<$ satisfying condition (3).

Let us consider again Examples  \ref{example-Student} and  \ref{example-Student-new} above.

\begin{example} 
Let us consider the TBox in Example \ref{example-Student}:
\begin{quote}
 $\mathit{ \tip(Student) \sqsubseteq \neg Pay\_Taxes}$\\
 $\mathit{ \tip(WStudent) \sqsubseteq Pay\_Taxes}$\\
 $\mathit{\tip(Student) \sqsubseteq  Smart}$\\
 $\mathit{WStudent \sqsubseteq  Student}$
\end{quote}
In all minimal canonical BP-models $\emme$, $k_{\emme,rc}(\mathit{Student})=0$, while $k_{\emme,rc}(\mathit{WStudent})$
$=k_{\emme,rc}$ $(\mathit{WStudent \sqcap Smart})=$ $k_{\emme,rc}(\mathit{WStudent \sqcap Smart})=1$,
as in the model of the rational closure.
Let $x$ and $y$ be two elements in the domain of $\emme$ such that: $k_{\emme,rc}(x)= k_{\emme,rc}(y)=1$,
 $x \in \mathit{WStudent \sqcap Pay\_Taxes \sqcap Smart}$,  and $y \in $ {\em WStudent}  $\mathit{\sqcap Pay\_Taxes \sqcap \neg Smart}$.
Such elements $x$ and $y$ exist in $\emme$ as $\emme$ is canonical.
As $y$ violates the typicality inclusion $\mathit{\tip(Student) \sqsubseteq  Smart}$, which is satisfied by $x$, and there is no typicality inclusion which is satisfied by $y$ and violated by $x$, by condition (3) in Definition  \ref{multipreference-model-of-K}, it must be that $x <y$.

Hence, in all the minimal canonical models $\emme$ of  the KB, the domain elements $z$ which are instances of $\tip(\mathit{WStudent})$
(and hence must have rank $k_{\emme,rc}(\mathit{z})=1$), not only must be instances of  $\mathit{WStudent \sqcap Pay\_Taxes}$
(as the defeasible inclusion $\mathit{ \tip(WStudent) \sqsubseteq Pay\_Taxes}$ must be satisfied by all the typical working student), but also
must be instances of $\mathit{WStudent \sqcap Pay\_Taxes \sqcap Smart}$, as they are preferred in $\emme$ to $ \mathit{WStudent \sqcap Pay\_Taxes \sqcap \neg Smart}$ elements. Therefore, $\mathit{ \tip(WStudent)}$ $\mathit{ \sqsubseteq Smart}$ holds in $\emme$.
\end{example}
In Example \ref{example-Student} entailment in minimal canonical  BP-models captures the defeasible inclusions which belong to the skeptical closure.
However, this is not the case in general.

\begin{example} \label{exa:SSN}
Let us consider, as a variant of  Example \ref{example-Student-new}, a knowledge base $K=({\cal T},{\cal A})$ with ${\cal A}=\emptyset$
and the following TBox ${\cal T}$:
\begin{quote}
 $\mathit{ \tip(Student) \sqsubseteq  Young}$\\
 $\mathit{ \tip(Student) \sqsubseteq \neg PayTaxes \sqcap  \exists hasSSN.\top}$\\
 $\mathit{ \tip(Employee) \sqsubseteq PayTaxes \sqcap  \exists hasSSN.\top}$\\
 $\mathit{\tip(Student \sqcap Employee) \sqsubseteq  \neg Young}$
\end{quote}
stating that typical students (and typical employee) have a social security number.
As in Example \ref{example-Student-new} in all the minimal canonical BP-model $\emme$ of $K$, we have  $k_{\emme,rc}(\mathit{Student)}$=
$k_{\emme,rc}$ $(\mathit{Employee)}=0$ and $k_{\emme,rc}(\mathit{Student  \sqcap}$ $\mathit{  Employee)}=1$,
as in the rational closure. As  $E_1 =  \mathit{Strict_{\cal T}  \cup \{ \tip(Student \sqcap}$ $\mathit{ Employee) \sqsubseteq  \neg Young} \}$,
% the concepts  $\mathit{Student}$ and $\mathit{Employee}$ have rank $0$, while concept  $\mathit{Student  \sqcap  Employee}$ has rank $1$.
%As we have seen for the KB in Example \ref{example-Student-new}, 
in the skeptical closure  construction:
\begin{align*}
S_0^B =\{ & \mathit{ \tip(Student)  \sqsubseteq \neg PayTaxes \sqcap  \exists hasSSN.\top,} \\
& \mathit{ \tip(Employee) \sqsubseteq PayTaxes \sqcap  \exists hasSSN.\top}\}
\end{align*}
and the set $S_0^B$ %contains conflicting defaults about tax payment, and
is not (globally) compatible with $ \mathit{Student \sqcap}$ $\mathit{  Employee}$ and $E_1$, so that $S^{sk,B} = E_1$.
Hence, $\tip(\mathit{Student  \sqcap Employee) \sqsubseteq  \exists hasSSN.\top}$ is not in the skeptical closure of the KB.
However, it is easy to see that this defeasible inclusion is satisfied in all the minimal canonical  BP-models $\emme$ of $K$.
i.e.,  $K \models_{BP}^{min} \tip(\mathit{Student  \sqcap Employee) }$ $\mathit{\sqsubseteq  \exists hasSSN.\top}$.

To see why $K \models_{BP}^{min} \tip(\mathit{Student  \sqcap Employee) \sqsubseteq  \exists hasSSN.\top}$, let $\emme = \langle \Delta, <_{rc}, <, I \rangle$ be a minimal canonical  BP-model of $K$ and let 
\begin{align*}
y \in \tip((\mathit{Student  \sqcap Employee))^I} &= min_{<}(\mathit{Student  \sqcap Employee})^I\\
& \subseteq min_{<_{rc}}(\mathit{Student} \mathit{ \sqcap Employee})^I
\end{align*}
(the last inclusion holds by Corollary  \ref{corollary}). 
We show that $y \in (\mathit{\exists hasSSN.\top)^I}$.
By contradiction, suppose that $y \not \in (\mathit{\exists hasSSN.\top)^I}$. As $\mathit{y \in (Student \sqcap \neg \exists hasSSN.\top)^I}$ $y$ violates both the second and the third defeasible inclusions in ${\cal T}$.
In the canonical model $\emme$ there must be an element $x \in min_{<_{rc}}(\mathit{Student  \sqcap Employee})^I$ such that
$x \in \mathit{(PayTaxes \sqcap  \exists hasSSN.\top)^I}$, so that $x$ does not violate the second  defeasible inclusion  $\mathit{ \tip(Student) \sqsubseteq \neg PayTaxes \sqcap  \exists hasSSN.\top}$, which is violated by $y$.
Also, $x$ satisfies  the inclusions in $E_1$, so that there is no inclusion which is violated by $x$ and not by $y$. Hence,
$x<y$ must hold in $\emme$, by condition (3) of Definition  \ref{multipreference-model-of-K},
which contradicts the hypothesis that $y \in \tip((\mathit{Student  \sqcap Employee))^I}$.
\end{example}

The example above shows that entailment in minimal canonical  BP-models is too strong for providing a characterization of the skeptical closure:
$\tip(\mathit{Student  \sqcap Employee)}$ $\mathit{ \sqsubseteq  \exists hasSSN.\top}$ is minimally entailed by $K$, but it is not in the skeptical closure of $K$. 
In the next section we consider a stronger closure construction, which is characterized by minimal canonical BP-models
%that minimal BP-entailment corresponds to a stronger closure construction which was proposed in \cite{Multipref_arXiv2017}. 
and, from this result, in Section \ref{sec:semantics_Skeptical_closure} we can provide a semantics for the skeptical closure. % building on the semantics of minimal canonical  BP-models.

  \section{Correspondence between BP-models and  a variant of lexicographic closure} \label{sez:lex_closure}
  
In this section we show that the semantics of minimal canonical BP-models introduced in the previous section provides a characterization of the 
multipreference closure  (MP-closure, for short), introduced in \cite{Multipref_arXiv2017} as a variant of the lexicographic closure  \cite{Lehmann95,Casinistraccia2012}.
%by Lehmann  \cite{Lehmann95}. % and extended to $\alc$ by Casini and Straccia \cite{Casinistraccia2012}.
More precisely, the MP-closure has been show to provide a sound approximation of the multipreference semantics introduced in  \cite{Multipref_arXiv2017}, 
a refinement of the rational closure semantics to cope with the ``all or nothing" problem.

%The variant we consider, called multipreference closure (MP-closure, for short), has been introduced in \cite{Multipref_arXiv2017} as an approximation of a multipreference semantics, an enhancement of the refinement of rational closure semantics proposed in \cite{GliozziAIIA2016}.
In the following we recap the definition of MP-closure and we prove that the typicality inclusions which hold in the MP-closure are those entailed from the KB under the minimal canonical BP-models semantics defined in section \ref{sez:BPsem}, which thus provides a sound and complete characterization of the MP-closure.

Let $B$ be a concept with rank $k$. Informally, we want to consider all the possible maximal sets of typicality inclusions $S$ which are compatible with $E_k$ and with $B$, i.e. the maximal sets of defeasible properties that a $B$ element can enjoy besides those in $E_k$.
For instance, in Example \ref{example-Student-new}, if $\mathit{B=Student  \sqcap  Employee}$, with $\rf(B)=1$, we have two possible alternative ways of maximally extending the set $E_1$, containing the defeasible inclusion $ \mathit{\tip( Student \sqcap Employee) \sqsubseteq \neg Young}$:  either with the defeasible inclusion  $\mathit{ \tip(Student) \sqsubseteq \neg PayTaxes}$
or with the defeasible inclusion $ \mathit{\tip( Employee) \sqsubseteq }$ $\mathit{PayTaxes}$.
As we have seen in Example \ref{example-Student-new}, these two defeasible inclusions are conflicting, and in the skeptical closure we do not accept any of them. However, 
here we consider all alternative maximally consistent scenarios, compatible with the fact that the concept $\mathit{B=Student  \sqcap}$ $\mathit{  Employee}$
is nonempty. In none of these scenarios the defeasible property that normally students are young %($\mathit{ \tip(Student) \sqsubseteq  Young}$) 
can be accepted, as it is incompatible with the more specific property that normally students which are epmployee are not young.

 Let  $\delta(E_i)$ be the set of typicality inclusions contained in $E_i$ (i.e. those defeasible inclusions with rank $\geq i$) and
 let $D_i= \delta(E_i) -\delta(E_{i+1})$  be the set of typicality inclusions with rank $i$.
% and $D_0,D_1, \ldots, D_n$ be such that 
Observe that $\delta(E_0)= \delta({\cal T})$.
Given a set $S$ of typicality inclusions of the TBox,  we let: $S_i= S \cap D_i$, for all ranks $i=0,\ldots, n$ in the rational closure, 
thus defining a partition of the typicality inclusions with finite rank in $S$, according to their rank\footnote{Observe that, we can ignore the defeasible inclusions with infinite rank when we consider a set of defaults maximally compatible with a concept $B$ (with rank $k$) and with $E_K$, as all the defeasible inclusions with infinite rank already belong to $E_k$}.
We introduce a preference relation among sets of typicality inclusions as follows:
$S' \prec S$ ({\em $S'$ is preferred to $S$})  if and only if
there is an $h$ such that, $S_h\subset S'_h$ and, for all $j > h$, $S'_j=S_j$.
The meaning of $S' \prec S$ is that, considering the highest rank $h$ in which $S$ and $S'$ do not contain the same defeasible inclusions,
 $S'$ contains more defeasible inclusions in $D_h$ than $S$.
 
 The preference relation $\prec$ introduced above differs from the one used in the  lexicographic closure as the lexicographical order in \cite{Lehmann95,Casinistraccia2012} %compares the size of $S_h$ and $S'_h$ for each rank.
considers the size of the sets of defaults for each rank.
%$S_h$ and $S'_h$, for each rank $h$. %(replacing condition $|S_h| <|S'_h|$, to the subset inclusion $S_h\subset S'_h$).
Here, 
the comparison of the sets of defeasible inclusions with the same rank %($S_i$ and $S'_i$) 
is based on subset inclusion ($S_h\subset S'_h$) and on equality among sets ($S'_j=S_j$) rather than on comparing the size of the sets (e.g., $|S_h| <|S'_h|$ or $|S'_j|=|S_j|$), as in the lexicographic closure.
%The notion of  lexicographic closure has been extended to the description logic $\alc$ by Casini and Straccia in \cite{Casinistraccia2012}.
For this reason,  the partial order relation $\prec$ is not necessarily modular,
which fits with the fact that in  BP-interpretations, the partial order relation $<$ is not required to be modular.

 \begin{definition}[ \cite{Multipref_arXiv2017}]
Let $B$ be a concept such that $\rf(B)=k$ and
let $S \subseteq \delta(TBox)$.
$ S  \cup E_k $ is a  {\em maximal set of defeasible inclusions compatible with $B$ in $K$}  %for rank $k$ }
if:
\begin{itemize}
\item
%$ S  \cup E_k \not \models_{\alctr} \tip(\top) \sqsubseteq \neg B$
$ E_k \not \models_{\alctr} \tip(\top) \cap \tilde{S} \sqsubseteq \neg B$
\item
and there is no $S' \subseteq \delta(TBox)$
such that %$ S'  \cup E_k \not \models_{\alctr} \tip(\top) \sqsubseteq \neg B$ 
$E_k \not \models_{\alctr} \tip(\top) \cap \tilde{S'} \sqsubseteq \neg B$ and $S' \prec S$ ($S'$ is preferred to $S$).
\end{itemize}
where $\tilde{S}$ is the materialization of $S$, i.e.,  $\tilde{S}= {\sqcap} \{ (\neg C \sqcup D) \mid  \tip(C) \sqsubseteq D \in S \}$.

%where we let:
%$S' \prec S$ if and only if
%there is an $h$ such that, $(S \cap D_h) )\subset (S' \cap D_h)$ and, for all $j > h$, $(S' \cap D_j)=(S \cap
%D_j)$.
\end{definition}
Informally, $S$ is a maximal set of defeasible inclusions compatible with $B$ and $E_k$ if
%it is not possible to add other defeasible inclusions to $S$ without excluding all $B$-elements, and 
there is no set $S'$ which is consistent with $E_k$ and $B$ and is preferred to $S$ since it contains more specific defeasible inclusions.
The construction is  similar to the lexicographic closure \cite{Lehmann95,Casinistraccia2012}, although, in this case, 
the lexicographic order $\prec$ is different, and it is easy to see that the MP-closure is weaker than the lexicographic closure 
(see Example \ref{exa:lexico-stronger-MP} below).

To check if a subsumption $\tip(B) \sqsubseteq D$ is derivable from the MP-closure of TBox 
we consider all the maximal sets of defeasible inclusions  $S$ that are compatible with  $B$. 
\begin{definition}[ \cite{Multipref_arXiv2017}]
%Let  $\tip(B) \sqsubseteq D$ be a query and let $k=\rf(B)$ be the rank of concept $B$ in the rational closure of ${\cal T}$.
A query $\tip(B) \sqsubseteq D$ {\em follows from the MP-closure of ${\cal T}$} if  
either the rank of concept $B$ in the rational closure of ${\cal T}$ is infinite or $\rf(B)=k$ is finite and
for all the maximal sets of defeasible inclusions  $S$ that are compatible with $B$ in $K$, 
we have:
$$E_k \models_{\alctr} \tip(\top) \sqcap \tilde{S} \sqsubseteq (\neg B \sqcup D)$$
\end{definition}

\normalcolor
Verifying whether a query $\tip(B) \sqsubseteq D$ 
%holds in all the minimal canonical enriched  models of the TBox
is derivable from the MP-closure of the TBox
in the worst case requires to consider a number of
%of $\{ E^{A_1}_{m^1}, \ldots, E^{A_r}_{m^r} \}$,
%to find the maximal $\Delta$ such that $\bigcup \Delta \cup E_k \models_{\alctr} \tip(\top) \sqsubseteq (\neg B \sqcup D)$.
maximal subsets $S$ of defeasible inclusions compatible with  $B$ and $E_k$, which is exponential in the number of typicality inclusions in $K$.
%Hence computing MC-closure has the same complexity as computing the lexicographic colsure
%(although in the last case, the bases are compared using the lexicographic ordering).
%\marginpar{\color{blue} verificare}
As entailment in $\alctr$ can be computed in  \textsc{ExpTime} \cite{AIJ15}, this complexity is still in  \textsc{ExpTime}.
%\marginpar{\color{blue} Mi pare rimanga Exptime - $2^n$ ciamate di costo $O(2^n)$ - verificare con complessita' Abox reasoning in AIJ15}. 
However, in practice, it is clearly less effective than computing subsumption in the skeptical closure of TBox,
which only requires a polynomial number of calls to  entailment  checks in $\alctr$, which can be computed by a linear encoding of an $\alctr$ KB into $\alc$ \cite{ISMIS2015}.

\begin{example} \label{exa:SSN}
Let us consider again the knowledge base $K=({\cal T},{\cal A})$ of  Example \ref{exa:SSN},  with ${\cal A}=\emptyset$
and the following TBox ${\cal T}$:
\begin{quote}
 $\mathit{ \tip(Student) \sqsubseteq  Young}$\\
 $\mathit{ \tip(Student) \sqsubseteq \neg PayTaxes \sqcap  \exists hasSSN.\top}$\\
 $\mathit{ \tip(Employee) \sqsubseteq PayTaxes \sqcap  \exists hasSSN.\top}$\\
 $\mathit{\tip(Student \sqcap Employee) \sqsubseteq  \neg Young}$
\end{quote}
We have seen that the typicality inclusion $\tip(\mathit{Student  \sqcap Employee) \sqsubseteq  \exists hasSSN.\top}$ is not in the skeptical closure of ${\cal T}$, but it holds in all the minimal canonical BP-models of $K$.
We can see that $\tip(\mathit{Student  \sqcap Employee) \sqsubseteq  \exists hasSSN.\top}$ follows from the MP-closure of TBox ${\cal T}$.
In fact, in this example there are two maximal sets of defeasible inclusions compatible with $B= \mathit{Student  \sqcap Employee}$ (where $\rf(B)=1$):
\begin{align*}
%E_1 = & \mathit{Strict_{\cal T}  \cup \{ \tip(Student \sqcap Employee) \sqsubseteq  \neg Young} \} \\
S =\{ & \mathit{ \tip(Student)  \sqsubseteq \neg PayTaxes \sqcap  \exists hasSSN.\top,  \tip(Student \sqcap Employee) \sqsubseteq  \neg Young} \}\\
S' =\{ & \mathit{ \tip(Employee) \sqsubseteq PayTaxes \sqcap  \exists hasSSN.\top,  \tip(Student \sqcap Employee) \sqsubseteq  \neg Young}\}
\end{align*} 
where $S$ is partitioned, according to the ranks of defaults, as follows:
\begin{align*}
S_0 = &\{  \mathit{ \tip(Student)  \sqsubseteq \neg PayTaxes \sqcap  \exists hasSSN.\top}\} \\
S_1= & \{ \mathit{\tip(Student \sqcap Employee) \sqsubseteq  \neg Young}\} \\
S_2 = & \emptyset
\end{align*} 
and $S'$ is partitioned as follows:
\begin{align*}
S'_0 = &\{  \mathit{  \tip(Employee) \sqsubseteq PayTaxes \sqcap  \exists hasSSN.\top}\} \\
S'_1= & \{  \mathit{\tip(Student \sqcap Employee) \sqsubseteq  \neg Young}\} \\
S'_2 = & \emptyset
\end{align*} 
Observe that neither $S\prec S'$ nor $S' \prec S$ and hence both $S$ and $S'$ are maximal sets of defeasible inclusions compatible with $B$.
In this case, $S$ and $S'$ would also correspond to the bases of the lexicographic closure of the KB.
%As $S_0$ and $S'_0$ are incomparable (neither $S_0 \subseteq S'_0$ nor  $S'_0 \subseteq S_0$) then
\end{example}
We refer to  \cite{Multipref_arXiv2017} for further examples concerning the MP-closure.
Before showing the correspondence between the MP-closure and BP-semantics,
let us show an example in which the lexicographic closure allows  conclusions which are not in the MP-closure.
\begin{example} \label{exa:lexico-stronger-MP}
If we modify the knowledge base in Example \ref{exa:SSN} above, by adding to the TBox the typicality inclusion
$\mathit{ \tip(Student) \sqsubseteq \neg PayTaxes \sqcap  Smart}$
we would get again two maximal sets of defeasible inclusions compatible with $B= \mathit{Student  \sqcap Employee}$ in the MP-closure construction:
\begin{align*}
%E_1 = & \mathit{Strict_{\cal T}  \cup \{ \tip(Student \sqcap Employee) \sqsubseteq  \neg Young} \} \\
S =\{ & \mathit{ \tip(Student)  \sqsubseteq \neg PayTaxes \sqcap  \exists hasSSN.\top,} \; \mathit{ \tip(Student) \sqsubseteq \neg PayTaxes \sqcap  Smart,} \\
& \mathit{ \tip(Student \sqcap Employee) \sqsubseteq  \neg Young} \}\\
S' =\{ & \mathit{ \tip(Employee) \sqsubseteq PayTaxes \sqcap  \exists hasSSN.\top,} \;  \mathit{\tip(Student \sqcap Employee) \sqsubseteq  \neg Young}\}
\end{align*} 
However, only  $S$ corresponds to a base in the lexicographic closure, %as $|S'_1| <|S_1|$.
as $S$ contains two defaults with rank 1, while $S'$ contains just one default of rank $1$ (and both $S$ and $S'$ contain the same number of defaults of rank $2$).
%NOTA: Se $S_h \not \subseteq S'_h$ and $S'_h \not \subseteq S_h$, ma $|S_h|=|S'_h|$, $S$ e $S'$ sono inconfrontabili nella MP closure, ma possono essere confrontabili nella lexicographic closure.
\end{example}
%%%%%  FINE MP CLOSURE    %%%%%%%%% 
%%%%%.  MP CLOSURE E BP SEMANTICS    %%%%%%%%

To show that the typicality inclusions derivable form the MP-closure of the KB
are exactly those that hold in all the minimal canonical BP-models of the KB, we prove the following two propositions.
The next one shows that  the MP-closure is sound with respect to the minimal canonical BP-semantics: 
If $\tip(B) \sqsubseteq D$ follows from the MP-closure of TBox, then $TBox \models_{BP}^{min} \tip(B) \sqsubseteq D$.
Let us prove the contrapositive.

\begin{proposition}  \label{correttezza_MP_closure}
Let ${\cal T}$ be a TBox and $B$ a concept with  $\rf(B)=k$ a finite rank in the rational closure construction. 
If there is a minimal canonical BP-model $\emme = \langle \Delta,<_{rc}, <, I \rangle$ of ${\cal T}$ and an element $x \in \Delta$ such that $x \in min_<(B^I) \sqcap \neg D$,
then there is a maximal set of defeasible inclusions $S$ compatible with $B$ in ${\cal T}$, such that 
$$ E_k \not \models_{\alctr} \tip(\top) \sqcap \tilde{S} \sqsubseteq (\neg B \sqcup D)$$
\end{proposition}
\begin{proof}
Assume that for some minimal canonical BP-model $\emme = \langle \Delta,<_{rc}, <, I \rangle$ of $K$ there is an element $x \in \Delta$ such that $x \in min_<(B^I) \sqcap \neg D$.
Let us define $S$ as the set of all the defeasible inclusions in TBox which are satisfied in $x$, i.e.

$S = \{ \tip(C) \sqsubseteq E \in \mbox{TBox} \mid \; x \in ( \neg C \sqcup E) \}$. % \mbox{ and } \rf(C) \mbox{ is finite }

\noindent
We show that, $ E_k \not \models_{\alctr} \tip(\top) \sqcap \tilde{S} \sqsubseteq (\neg B \sqcup D)$.

Let $\emme^{RC}= \langle \Delta,<_{rc}, I \rangle$. By construction,  $\emme^{RC}$ is a minimal canonical model of the rational closure of $K$.
By a property of $\emme^{RC}$ (Proposition 12 in \cite{AIJ15}), $\emme_k^{RC}$ (i.e. the model obtained by $\emme$ by collapsing all the element with rank $\leq k$ to rank $0$) satisfies $E_k$: $\emme_k^{RC} \models_{\alctr} E_k$.
Also, as $\rf(B)=k$ and  $x \in \tip(B)^I$, $x$ must have rank $k$ in $\emme^{RC}$, and hence rank $0$ in $\emme_k^{RC}$
(and, clearly, $k_{\emme,rc}(x)=k$ in $\emme$).  
 Thus, $x \in \tip(\top)^I$ holds in $\emme_k^{RC}$, but also $x \in (B \sqcap \tilde{S})^I$ (by definition of $S$). 
 Therefore
$\emme_k^{RC} \not \models_{\alctr} \tip(\top) \sqcap \tilde{S} \sqsubseteq \neg B$.
Hence,
 $E_k \not \models_{\alctr} \tip(\top)  \sqcap \tilde{S}  \sqsubseteq \neg B$,
 i.e. $S$ is a set of defeasible inclusions compatible with $B$.
 
 Furthermore, as $x \in \neg D$,  $E_k \not \models_{\alctr} \tip(\top)  \sqcap \tilde{S} \sqcap \neg D \sqsubseteq \neg B$,
 and hence, $ E_k \not \models_{\alctr} \tip(\top) \sqcap \tilde{S} \sqsubseteq (\neg B \sqcup D)$, i.e.,
 $\tip(B) \sqsubseteq D$ does not follow from the MP-closure of TBox.
 
To show that $S$ is a maximal set of defeasible inclusions compatible with $B$,
we have still to show that $S$ is maximal.
Suppose, by contradiction, it is not. Then there is a set $S'$ such that $S' \prec S$ and  
$E_k \not \models_{\alctr} \tip(\top)  \sqcap \tilde{S'}  \sqsubseteq \neg B$.
Therefore, there must be a $\alctr$ model $\enne= \langle \Delta',<'_{rc} I' \rangle$ of  $E_k$ and an element $y \in \Delta'$, having rank $0$ in $\enne$ such that:
$y \in (\tilde{S'} \sqcap B)^{I'}$. 
%We could extend  $\emme^{RC}$ by adding a new element $y$ in $\Delta$ with rank $0$ and with the same interpretation of atomic concepts as in I'.
%together with all the domain elements $\Delta_y= \{ w : y (\bigcup R_i)^+ w \}$ reachable from $y$ in $\enne$.

As $\emme$ is canonical, then $\emme^{RC}$ is canonical as well. Hence, there must be an element $z \in \Delta$ such that $z \in (\tilde{S'} \sqcap B)^I$
(i.e., the interpretation of all non-extended concepts in $z$ is the same as in $y$ in $\enne$).
As $y$ has rank $0$ in $\enne$, $y$ satisfies all the defeasible inclusions in $E_k$.
Hence, the concept $\tilde{S'} \sqcap B$ must have rank $k$ in the rational closure and, therefore, $z$ must have rank $k$ in $\emme^{RC}$.
%and rank $0$ in $\emme_i^{RC}$. 
Thus, $z \in (\tip(\top) \sqcap \tilde{S'} \sqcap B)^I$ in $\emme_k^{RC}$,
%and  $\emme_k^{RC} \not \models_{\alctr} \tip(\top) \sqcap \tilde{S'} \sqsubseteq \neg B$. Thus,
%$ E_k \not \models_{\alctr} \tip(\top) \sqcap \tilde{S'} \sqsubseteq (\neg B \sqcup D)$.
and, clearly, $k_{\emme,rc}(z)=k$ in $\emme$.

Since $S' \prec S$ there must be some 
$h$ such that, $S_h\subset S'_h$ and, for all $j > h$, $S'_j=S_j$.
Thus, there is some defeasible inclusion $\tip(C') \sqsubseteq E' \in S'$ such that $\tip(C') \sqsubseteq E' \not \in S$.
so that $z$ satisfies $\tip(C') \sqsubseteq E'$ (i.e., $z \in (\neg C' \sqcup  E')^I$, while $x$ violates it (i.e., $x \in (C' \sqcap \neg E')^I$).
On the other hand, all the defeasible inclusion violated by $z$ and not by $x$ cannot have rank $\geq h$, as $x$ satisfies %exactly 
only the inclusions $S$ (by definition of $S$) and  
and, for all $j \geq h$, $S_j'= S_j$ (the typicality inclusions with infinite rank are trivially satisfied both in $x$ and in $z$).

Therefore, $z <x$ holds in $\emme$ by condition (3), and $x$ cannot be a typical $B$ element, thus contradicting the hypothesis.
%$\hfill \bbox$
\end{proof}

The next proposition shows that  the MP-closure is complete with respect to the minimal canonical BP-semantics: 
If $TBox \models_{BP}^{min} \tip(B) \sqsubseteq D$, then $\tip(B) \sqsubseteq D$ follows from the MP-closure of TBox. 
Let us prove the contrapositive.

\begin{proposition} \label{completezza_MP_closure}
${\cal T}$ be a TBox and $\tip(B) \sqsubseteq D$ a defeasible inclusion
such that $\rf(B)=k$ is a finite rank in the rational closure. 
If $\tip(B) \sqsubseteq D$ does not follow from the MP-closure of ${\cal T}$,
%If there is a maximal set of defeasible inclusions $S$ compatible with $B$ in $K$, such that 
%$$ E_k \not \models_{\alctr} \tip(\top) \sqcap \tilde{S} \sqsubseteq (\neg B \sqcup D)$$
then there is a minimal canonical MP model $\emme = \langle \Delta,<_{rc}, <, I \rangle$ of ${\cal T}$ and an element $x \in \Delta$ such that $x \in min_<(B^I) \sqcap \neg D$.
\end{proposition}
\begin{proof}
If $\tip(B) \sqsubseteq D$ does not follow from the MP-closure of ${\cal T}$,
then there is a maximal set of defeasible inclusions $S$ compatible with $B$ in $K$, such that 
$$ E_k \not \models_{\alctr} \tip(\top) \sqcap \tilde{S} \sqsubseteq (\neg B \sqcup D).$$
Then
$$ E_k \not \models_{\alctr} \tip(\top)  \sqsubseteq \neg ( \tilde{S} \sqcap B \sqcap \neg D)$$
and concept $\tilde{S} \sqcap B \sqcap \neg D$  is not exceptional with respect to $E_k$
and, in the rational closure, it must have rank less or equal to $k$.
As $\rf(B)=k$, it must be  $\rf(\tilde{S} \sqcap B \sqcap \neg D)=k$.

Let us consider any minimal canonical $\alctr$ model $\enne  = \langle \Delta',<_{RC}, I' \rangle$ of $K$.
As $\rf(\tilde{S} \sqcap B \sqcap \neg D)=k$, by Proposition 13 in \cite{AIJ15}, 
the concept $\tilde{S} \sqcap B \sqcap \neg D$ must have rank $k$ in any minimal canonical model of $K$. 
Therefore,
$k_{\enne}(\tilde{S} \sqcap B \sqcap \neg D)=k$,
and there is an element $y \in \Delta$ such that $y \in (\tilde{S} \sqcap B \sqcap \neg D)^{I'}$ and $k_\enne(y)=k$.

From $\enne$ we build a minimal canonical MP model $\emme = \langle \Delta,<_{rc}, <, I \rangle$  falsifying $\tip(B) \sqsubseteq D$
as follows.
We let $\Delta= \Delta'$, $I=I'$ and $<_{rc}= <_{RC}$.
We define $<$ as the transitive closure of $<^1$, where
$x <^1 y$ is true if and only if the antecedent of condition (3) in Definition \ref{multipreference-model-of-K} holds, that is:
\begin{quote} 
$x<^1 y$ {\bf if and only if}\\
$\;$ \mbox{\ \ \ \ }  there is some $\tip(C) \sqsubseteq F \in K$ which is violated by $y$ and, \\
$\;$ \mbox{\ \ \ \ }  for all $\tip(C_j) \sqsubseteq D_j \in K$, which is violated by $x$ and not by $y$, \\
$\;$ \mbox{\ \ \ \ }   there is a $\tip(C_k) \sqsubseteq D_k \in K$,  which is violated by $y$ and not by $x$,  and \\
$\;$ \mbox{\ \ \ \ }  $\;$ \mbox{ \ \ \ \ \ \ \ \ \ \ \ \ \ \ \ \ \ \ \ \ \ \ \ \ \ \ \ \ \ \ \ \ \ \ \ \ \ \ \ \ \ \ \ \ \ \ \ \ } $k_{\emme,rc}(C_j) < k_{\emme,rc}(C_k)$.
\end{quote}
Observe that, for all concepts $C$, $k_{\emme,rc}(C)=k_{RC}(C)= \rf(C)$, the rank of $C$ in the rational closure.
We have to show that $\emme$ is a minimal canonical MP model of $K$ and that $y \in (\tip(B) \sqcap \neg D)^I$.

We first show that $\emme$ is an MP model of $K$,  that it is canonical and that it is minimal among the canonical MP models of $K$.
To show that $\emme$ is an MP model of $K$, we observe that, by definition of $<$, condition $(3)$ in Definition \ref{multipreference-model-of-K} holds for $\emme$ by construction.

It can be easily seen that $\emme$ satisfies the assertions in  ABox and the strict inclusions $C \sqsubseteq E$ in TBox, since
$\enne$ does, $\Delta= \Delta'$ and $I=I'$.
To show that $\emme$ is an MP model of $K$, we have also to show that 
for all $\tip(C) \sqsubseteq {E}$ in TBox,
 $min_{<_{rc}}({C}^{I}) \subseteq {E}^{I}$ holds.
It follows from the fact that $min_{<_{RC}}({C}^{I'}) \subseteq {E}^{I'}$ holds  in $\enne$
and that, by definition of $\emme$, $<_{rc}=<_{RC}$ and $I=I'$.

We show that $\emme$ is a canonical BP model of $K$:
If not, there are $C_1, C_2, \dots, C_n$ such that
$K \not\models_{\alctrbp} C_1 \sqcap C_2 \sqcap \dots \sqcap C_n \sqsubseteq \bot$, but there is no $x \in \Delta$ such that
$x \in (C_1 \sqcap C_2 \sqcap \dots \sqcap C_n)^{I}$.
By Theorem  \ref{equivalence-multiple3},
$K \not\models_{\alctr} C_1 \sqcap C_2 \sqcap \dots \sqcap C_n \sqsubseteq \bot$
This would contradict the hypothesis that $\enne$ is an $\alctr$ canonical model of $K$.

We have to show that  $\emme$ is minimal among the canonical BP models of $K$.
If, by absurdum, $\emme$ were not a minimal canonical BP model, then there would be a BP model $\emme'' = \langle \Delta'',  <''_{rc}, <'', I''  \rangle$
in $Min_{RC}(K)$,
such that $\Delta'' = \Delta$, $I'' = I$,  and $\emme'' \prec_{BP} \emme$.
Observe that the relation $<''_{rc}$ in $\emme''$ must be equal to $<_{rc}$, as it is determined by a minimal canonical $\alctr$ model (and hence by the rational closure of TBox).

Concerning $<''$, as $\emme''$ is an BP interpretation, $<''$ must be transitive and contain $<^1$.
Hence, $<''$ must contain the transitive closure of $<^1$.
As $<$ is defined as the transitive closure of $<^1$, it must be $< \subseteq <''$,
which contradicts the hypothesis that that $\emme'' \prec_{BP} \emme$.

Finally, we want to show that  $y \in (\tip(B) \sqcap \neg D)^I$.
We have seen that in $\enne$ there is an element $y \in \Delta$ such that $y \in (\tilde{S} \sqcap B \sqcap \neg D)^{I'}$ and $k_\enne(y)=k$.
By construction of $\emme$, $I=I'$ and then $y \in (B \sqcap \neg D)^I$.
Furthermore, $<_{rc}=<_{RC}$ and, hence, $k_{\emme,rc}(y)=k_{\enne}(y)=k$ and, also, $k_{\emme,rc}(B)=k_{\enne}(B)=\rf(B)=k$.

To see that $y \in min_<(B)$, we need to show that there is no $z \in \Delta$ such that $z \in B^I$ and $z<y$.
Suppose by contradiction that there is such a $z$. As $z$ is a $B$-element, it cannot have rank less than $k$ in the rational closure.
Hence, it must be $k_{\emme,rc}(z)=k$.

Let $S'$ be the set of defeasible inclusions satisfied by $z$,
i.e., $S' = \{ \tip(C) \sqsubseteq E \in \mbox{TBox} \mid \; z \in ( \neg C \sqcup E)\}$ .
Then $z \in (\tilde{S'} \sqcap B)^I$.
Let $\emme^{RC}= \langle \Delta,<_{rc}, I \rangle$ be the $\alctr$ model obtained from $\emme$, ignoring the preference relation $<$.
By Proposition 12 in \cite{AIJ15}, $\emme_k^{RC}\models_{\alctr} E_k$ and, 
as $k_{\emme,rc}(z)=k$,  $z$ must have rank $0$ in  $\emme_k^{RC}$. Therefore,
$$ E_k \not \models_{\alctr} \tip(\top) \sqcap \tilde{S'} \sqsubseteq \neg B.$$

As $z <y$, for all defeasible inclusions $\tip(C_j) \sqsubseteq A_j \in K$ violated by $z$ and satisfied by $y$,
there is a more specific defeasible inclusion $\tip(C_k) \sqsubseteq A_k \in K$ violated by $y$ and satisfied by $z$ 
(that is $k_{\emme,rc}(C_j)<k_{\emme,rc}(C_k)$).
Suppose that $j$ is the rank of the defeasible inclusion with highest rank violated by $z$
and that $h$ is the rank of the defeasible inclusion with highest rank violated by $y$.
It must be $j<h$. Therefore, $S_h \subset S'_h$ (as $z$ satisfies all the defeasible inclusions of rank $h$).
Therefore, $S'$ is preferred to $S$, $S' \prec S$.
However, this contradicts the hypothesis that $S$ is a maximal set of defeasible inclusions compatible with $B$ in $K$.
Therefore, a $z$ with $z<y$ cannot exist and $y \in \tip(B)^I$, so that $y \in (\tip(B) \sqcap \neg D)^I$.
%$\hfill \bbox$
\end{proof}

We can now establish a correspondence between the minimal canonical MP models semantics and the MP closure.
\begin{theorem} \label{th:BP-MP}
Given a knowledge base $K=({\cal T}, {\cal A})$ and a query $\tip(B) \sqsubseteq D$,
${\cal T} \models_{BP}^{min} \tip(B) \sqsubseteq D$
if and only if $\tip(C) \sqsubseteq D$ follows from the MP-closure of the TBox ${\cal T}$.
%$\Delta \cup E_k \models_{\alctr} \tip(\top) \sqsubseteq (\neg B \sqcup D)$, where $k=\rf(B)$. 
\end{theorem}
\begin{proof}
The proof of this result can be done by contraposition and is an easy consequence of Proposition  \ref{correttezza_MP_closure} and Proposition  \ref{completezza_MP_closure}.
Just observe that, for the ``If" part, when ${\cal T} \not \models_{BP}^{min} \tip(B) \sqsubseteq D$, concept $B$ must have a finite rank, %$\rf(B)\neq \infty$, 
otherwise $\tip(B) \sqsubseteq D$ would be a logical consequence of ${\cal T}$, for any concept $D$.
For the ``Onfy if" part, when $\tip(C) \sqsubseteq D$ does not follow from the MP-closure of the TBox ${\cal T}$, the rank of $B$ in the rational closure must be finite.
\end{proof}
In \cite{Multipref_arXiv2017} we have shown that the MP-closure provides a sound approximation of a multipreference semantics, the S-enriched semantics. From the correspondence result above (Theorem \ref{th:BP-MP}), it also follows that entailment with respect to the minimal canonical BP-models (as defined in Section \ref{sez:BPsem}) is strictly weaker than entailment with respect to the minimal canonical S-enriched models defined in \cite{Multipref_arXiv2017}.

 \section{A semantic characterization for the skeptical closure} \label{sec:semantics_Skeptical_closure}

First we show that we can equivalently reformulate the notion of global compatibility of a set of defeasible inclusions (Definition \ref{def:global_compatibility}), as stated by the following property:
\begin{proposition}
Let ${\cal T}$ be a TBox and $B$ be a concept with finite $\rf(B)=k$.
Given two sets of defeasible inclusions $S$ and $S'$,
 {\em $S$ is (globally) compatible with $B$ w.r.t.  $E_k \cup S'$} if and only if
$$ E_k    \not \models_{\alctr} \tip(\top)  \sqcap \tilde{S} \sqcap  \tilde{S'} \sqsubseteq \neg B $$
where $\tilde{S}$ is the materialization of $S$, i.e.,  $\tilde{S}= {\sqcap} \{ (\neg C \sqcup D) \mid  \tip(C) \sqsubseteq D \in S \}$.
\end{proposition}
\begin{proof}
%{\color{blue} Prova basata sul fatto che i defaults con rango infinito sono in $E_k$.}
Remember that
$E_k \subseteq {\cal T}$ is the set of defeasible inclusion having rank $\geq k$ in the rational closure construction.
%(and includes the defeasible inclusions with infinite rank).
We show that,  for any set $H$ of defeasible inclusions in ${\cal T}$:
\begin{center}
$ E_k \cup H    \models_{\alctr} \tip(\top)  \sqsubseteq \neg B $ $\iff$ 
$ E_k     \models_{\alctr} \tip(\top)  \sqcap \tilde{H} \sqsubseteq \neg B $.
\end{center}

\noindent
($\Leftarrow$)
By contraposition, suppose $ E_k \cup H  \not  \models_{\alctr} \tip(\top)  \sqsubseteq \neg B $.
Then,  there is an $\alctr$ model $\emme=\la \Delta, <, I\ra$ of $E_k \cup H$, and a domain element $x \in \Delta$ such that  
$k_{\emme}(x)=0$ \normalcolor and $x \in  B^I$.

We show that $x \in \tilde{H}^I$. Let us consider any typicality inclusion $\tip(C) \sqsubseteq D$ in $H$.
We show that $x$ is an instance of its materialization $\neg C \sqcup D$, i.e., $x \in (\neg C \sqcup D)^I$. 
If $x \not \in C^I$, the conclusion follows trivially.
If $x \in C^I$, the considering that $x$ has rank $0$ in $\emme$ and that  $\emme$ satisfies $\tip(C) \sqsubseteq D$,
$x$ is a typical $C$ element and hence it must be $x\in D^I$. Therefore, $x \in (\neg C \sqcup D)^I$.
As this holds for all the typicality inclusion in $H$, $x \in \tilde{H}^I$ and, hence, $x \in (\tip(\top) \sqcap \tilde{H} \sqcap B)^I$, which proves the thesis.

\noindent
($\Rightarrow$)  By contraposition, let $ E_k   \not  \models_{\alctr} \tip(\top)  \sqcap \tilde{H} \sqsubseteq \neg B $.
Then,  there is a model $\emme_1=\la \Delta_{1}, <_1, I_1\ra$ of $E_k$, and a domain element $x \in \Delta_1$ such that   $x \in (\tip(\top) \sqcap \tilde{H} \sqcap B)^{I_1}$, i.e., 
$k_{\emme_1}(x)=0$, $x \in \tilde{H}^{I_1}$ and $x \in  B^{I_1}$.
 
The model $\emme_1$ might not satisfy all the typicality inclusions $\tip(C) \sqsubseteq D$ in $H$. 
let us consider a model $\emme_2$ of $ E_k \cup H$. 
Such a model must exist, otherwise, the TBox ${\cal T}$ would be unsatisfiable and any concept would have an infinite rank in the rational closure of ${\cal T}$.
Conversely, we know that $B$ has a finite rank $k$.
Hence, let  $\emme=\la \Delta, <, I\ra$ be a finite minimal canonical model of $E_k \cup H$. Existence of a finite, minimal, canonical models of a consistent TBox in $\alctr$ is guaranteed by Theorem 7 in \cite{AIJ15}). Suppose that $\Delta$  and $\Delta_1$ are disjoint.
We build from $\emme$ and $\emme_1$  a new model $\emme'$ of $ E_k \cup H$ in which the concept $\tip(\top) \sqcap B$ is satisfiable.

Let us define $\emme'= \langle \Delta', <', I' \rangle$   as follows:
$\Delta'= \Delta \ \cup \ \Delta_1$; 
$I'$ is defined on the elements of $\Delta$ as $I$ in $\emme$, and on the elements of $\Delta_1$  as $I_i$ in $\emme_1$.
For the interpretation of concepts: for $x \in \Delta$, $x \in C^{I'}$ if and only if $x \in C^{I}$;
for $x \in \Delta_1$, $x \in C^{I'}$ if and only if $x \in C^{I_1}$.
For the interpretation of roles: for $x , y \in \Delta$, $(x,y) \in R^{I'}$ if and only if $(x,y) \in R^{I}$;
for $x,y \in \Delta_1$, $(x,y) \in R^{I'}$ if and only if $(x,y) \in R^{I_1}$; and, for any two elements $x \in \Delta$ and $y \in \Delta_1$, 
$(x,y) \not \in R^{I'}$ and $(y,x) \not \in R^{I'}$.
For all individual constants $a\in {\cal O}$, we let ${a}^{I'} = {a}^{I}$.
Finally,
%for all elements $x$ of $\Delta'$, we attribute them the lowest rank they can have (i.e.  $k_{\emme'}(w) =$ rank ($\Gamma$) where $\Gamma$ is the conjunction of all concepts of which w is an instance in $\emme'$.
for all $w \in \Delta$, we let $k_{\emme'}(w) =k_{\emme}(w)$, 
for the element $x \in \Delta_1$ which is an instance of  $\tip(\top) \sqcap \tilde{H} \sqcap B$, we let $k_{\emme'}(x) =0$;
finally,  for all $y \in \Delta_1$ ($y \neq x$), we let $k_{\emme'}(y) =n+1+k_{\emme_1}(y)$, where $n$ is the highest value of $k_{\emme}$ in $\emme$
($n$ is finite as each element in $\emme$ has a finite rank).
%In particular, $k_{\emme'}(x) =i$.
\normalcolor

It is easy to show that by construction the resulting model $\emme'$ satisfies $E_k \cup H$.
Let $C \sqsubseteq D$ be strict inclusion in $E_k \cup H$.
In the first case, $C \sqsubseteq D$ is a strict inclusion.
Let $x\in C^{I'}$. There are two cases: either $x \in \Delta$ or $x \in \Delta_1$.
In the first case, $x \in C^I$ in $\emme$. As $\emme$ satisfies $K$, $x \in D^I$ and, by definition of $\emme'$, $x \in D^{I'}$.
In the second case, $x \in C^{I_1}$. As $\emme_1$ satisfies all the strict inclusions in ${\cal T}$ (which belong to $E_k$), 
$x \in D^{I_1}$ and, by definition of $\emme'$, $x \in D^{I'}$.

Let $\tip(C) \sqsubseteq D$ be a defeasible inclusion in $E_k \cup H$.
If $\rf(C) \geq k$, then by the construction of the rational closure $\tip(C) \sqsubseteq D$ is in $E_k$
and hence is satisfied both in $\emme$ and in $\emme_1$.
Let $z\in (\tip(C))^{I'}$, then 
either $z \in \Delta$ or $z \in \Delta_1$.
In the first case, $z$ is $C$-minimal in $\emme$ and $z \in D^I$. Hence, by definition of $\emme'$, $z \in D^{I'}$.
In the second case,  $z$ is $C$-minimal in $\emme_1$ and $z \in D^{I_1}$. Hence, by definition of $\emme'$, $z \in D^{I'}$.

If $\rf(C) =j < k$, then $\tip(C) \sqsubseteq D$ is in $H$ but not in $E_k$. As the rank of $C$ in the rational closure is finite, by Proposition 13 in \cite{AIJ15},
$C$ has finite rank $j$ in any minimal canonical model of the TBox ${\cal T}$. Hence, $C$ is consistent with the TBox ${\cal T}$, as well as with its subset $E_k \cup H\subseteq  {\cal T}$. As $\emme$ is a canonical model of $E_k \cup H\subseteq  {\cal T}$, there must be an element in $w \in \Delta$ such that $w \in C^I$. Therefore, each minimal $C$ element in $\emme$ either is $x$ (and, in this case, $x$ is in $(\neg C \sqcap D)^{I'}$ and hence in $D^{I'}$),
or it is an element $z \in \Delta$. As $\emme$ satisfies $H$, it satisfies $\tip(C) \sqsubseteq D$ and, hence, $z \in D$.

%Observe that all the assertions in the ABox are satisfied in $\emme$
%and we have interpreted individual constants over the elements of $\Delta$ as in $\emme$:  $a^{I'}=a^I$, for all $a \in {\cal O}$.
%By construction, for $x \in \Delta$,  $x \in C^{I'}$ iff $x \in C^I$.
%Hence, if $B(a) \in$ ABox is satisfied in $\emme$, then it is satisfied in $\emme'$ as well.

From this, we can conclude that $\emme'$ is a model satisfying $E_k \cup H$,
which contains an element $x$ with rank $k_{\emme'}(x)=0$ such that $x \in B$.
Therefore, $ E_k \cup H  \not  \models_{\alctr} \tip(\top)  \sqsubseteq \neg B $, which concludes the proof.
\end{proof}
The  above reformulation of the notion of global compatibility makes the relationship between the notion of skeptical closure and the notion of MP-closure more evident.

In particular, for a concept $B$ with $\rf(B)=k$,
when (in the MP-closure construction) there is 
a single maximal set of defeasible inclusions $S$ compatible with $B$ in ${\cal T}$,
i.e., such that 
$ E_k \not \models_{\alctr} \tip(\top) \sqcap \tilde{S} \sqsubseteq \neg B$,
then $E_k \cup S$ corresponds to the skeptical closure $S^{sk,B}$ of ${\cal T}$ with respect to $B$.

When in the MP-closure there are different maximal sets of defeasible inclusions $S^1, \ldots, S^r$ compatible with $B$ in ${\cal T}$,
the skeptical closure is defined to contain, in addition to $E_k$, the defeasible inclusions with rank $j$ in $S^1, \ldots, S^r$, for those ranks $j$ from $h$ to $k-1$ on which $S^1, \ldots, S^r$ exactly agree
(i.e., $S^1_j= \ldots= S^r_j$), where $h-1$ is the highest rank on which  $S^1, \ldots, S^r$ disegree  (i.e., $S^l_{h-1} \neq S^m_{h-1}$, for some $l$ and $m$).
%and such that there is not a rank $j>h$ on which $S^1, \ldots, S^r$ disegree  (i.e., $S^l_j \neq S^m_j$).
If the sets $S^1, \ldots, S^r$ disagree on some defeasible inclusion with rank $j$, no defeasible inclusion with rank $j$ or lower is included in the skeptical closure.

%
%It is also clear that the notion of skeptical closure is weaker than the notion of MP-closure.

Based on the reformulation above and on the correspondence between the MP-closure of a knowledge base and its minimal canonical BP-models, we are now able to provide a semantic characterization of the skeptical closure.

Given a TBox ${\cal T}$, let $DI(B)$ be the set of the defeasible inclusions $\tip(C) \sqsubseteq D\in {\cal T}$ which are satisfied by all the minimal $B$ elements in any the minimal canonical BP-models of ${\cal T}$:
\begin{align*}
\mathit{DI(B)}= \{ \tip(C) \sqsubseteq D\in K \mid & \;  x\in (\neg C \sqcup D)^I, \mbox{ for any }  x \in min_<(B^I)  \mbox{ in any minimal  } \\
& \mbox{ canonical BP-model }   \emme= \la \Delta, <_{rc},<,\cdot^I \ra  \mbox{ of  ${\cal T}$  } \}
 \end{align*}
 %
%$DI_K^{MP}$ are the inclusions in $K$ which are BP-entailed by $K$.
%They are the defaults on which are accepted in all $B$ minimal elements in all the minimal canonical BP-model of $K$.
Let $\mathit{Confl\_DI(B)}$ be the set of the conflicting defeasible inclusions for $B$ in ${\cal T}$, defined as the typicality inclusions which are satisfied in some minimal $B$ element in a minimal canonical BP-models of ${\cal T}$, but not in all of them:
 \begin{quote}
$Confl\_DI(B)= \{ \tip(C) \sqsubseteq D\in K \mid \; x\in (\neg C \sqcup D)^I$ and $y\in ( C \sqcap \neg D)^I$ \\
$\;$ $\mbox{ \ \ \ \ \ \ \ \ \ \ \ \ \ \ \ \ \ \ \ \ \ \ \ \ \ \ \ \  }$for some minimal  canonical   BP-model $\emme= \la \Delta, <_{rc},<,\cdot^I \ra$ \\
$\;$ $\mbox{ \ \ \ \ \ \ \ \ \ \ \ \ \ \ \ \  \ \ \ \ \ \ \ \ \ \ \ \ }$of ${\cal T}$ and for some $x,y \in min_<(B^I) \}$
%\; \emme \models \tip(C) \sqsubseteq D,\; for all minimal \\ $\;$ $\mbox{ \ \ \ \ \ \ \ }$ \hspace{3cm}  canonical BP-models $\emme= \la \Delta, <_{rc},<,\cdot^I \ra$ of $K \}$
 \end{quote}
 They are the defaults on which there is no agreement among minimal $B$ elements in at least some minimal canonical BP-model of ${\cal T}$.
 Let ${\cal S}$ be all the concepts occurring in the knowledge base or in the query,
 and let   ${\cal C}_j$ be the set of all the concepts with rank $j$:
 \begin{align*}
{\cal C}_j= \{  C  \in {\cal S} \mid & \;  k_\emme(C^I)=j  \mbox{ in any  minimal  canonical  BP-model }  \\
& \mbox{ \ \ \ \ \ \ \ \ \ \ \ \ \ \ \ \ \ \ \ \ \ \ \ \ \ \ \ \ \ \ \ \ \ \ \ \ \ }\emme= \la \Delta, <_{rc},<,\cdot^I \ra  \mbox{ of  } {\cal T} \}
 \end{align*}
We identify the defeasible inclusions with rank $j$ in $\mathit{DI(B)}$ and in $\mathit{Confl\_DI(B) } $, respectively:
\begin{align*}  
 \mathit{DI_j(B)} = & \mathit{DI(B) } \cup {\cal C}_j \\
 \mathit{Confl\_DI_j(B) } = & \mathit{ Confl\_DI(B)} \cup {\cal C}_j
 \end{align*}
We can now define the set of defeasible inclusions which are included in the skeptical closure of $B$, $S^{sk,B}$, as follows:
$$\mathit{DI\_Sk(B)= \bigcup_ {j=h,k-1} DI_j (B) }$$
where $h$ is the lowest integer, form $0$ to $k-1$, such that, for all $j>h$,  $\mathit{Confl\_DI_j(B) } = \emptyset$.
%where $h$ is the highest integer, form $0$ to $k-1$, such that  $\mathit{Confl\_DI_h(B) } \neq \emptyset$.

 $\mathit{DI\_Sk(B)}$
is the set of defeasible inclusions which are included in the skeptical closure of $B$, $S^{sk,B}$.
 Essentially, $\mathit{DI\_Sk(B)}$ contains the defeasible inclusions on which all the minimal canonical models agree, in the following sense:
for each rank $j$, from $h$ to $k-1$, $\mathit{DI_j(B)}$ is the set of all the defeasible inclusions of rank $j$ which are satisfied by all the  minimal $B$-elements in all the minimal canonical BP-model of ${\cal T}$.
Also, the minimal $B$ elements of the minimal canonical BP-models of ${\cal T}$ 
must agree on accepting or not  the defeasible inclusions with rank $\geq h$ (as there are no conflicting defeasible inclusions of rank $\geq h$ for $B$).
Instead, they disagree on accepting or not some defeasible inclusion with rank $h-1$.
 
 \begin{proposition}

 Let  $\tip(B) \sqsubseteq D$ be a query and ${\cal T}$ a TBox. The defeasible inclusion
 $\tip(B) \sqsubseteq D$ is in the skeptical closure of TBox if and only if
$$\mathit{Strict({\cal T}) \cup DI\_Sk(B)} \models_{\alctr} \tip(\top) \sqsubseteq (\neg B \sqcup D)$$
where $\mathit{Strict({\cal T})}$ is the set of strict inclusions in ${\cal T}$.
 \end{proposition}

  \section{Conclusions and related work}  \label{sec:conclu} %(CONCLUSIONI DA FINALE PRUV)
 
% In this work we have defined a new notion of closure for the description logic ALC, which is a refinement of the rational closure,  weaker than the lexicographic closure, introduced by Lehmann  \cite{Lehmann95} and extended to the description logics by Casini and Straccia in \cite{Casinistraccia2012}.
We have introduced the skeptical closure which is  a weaker variant of the lexicographic closure  \cite{Lehmann95,Casinistraccia2012}, which deals with the problem of ``all or nothing" affecting the rational closure without generating alternative ``bases".
Its computation only requires a polynomial number of calls to the underlying preferential $\alctr$ reasoner.

 Other refinements of the rational closure, which also deal with this limitation of the rational closure, are the relevant closure \cite{Casini2014} and
 the inheritance-based rational closure \cite{Casinistraccia2011,CasiniJAIR2013}. 
 In particular, in  \cite{Casinistraccia2011,CasiniJAIR2013}, a new closure construction is defined by combining the rational closure with defeasible inheritance networks. This inheritance-based rational closure, in  Example  \ref{exa:weak}, is able to conclude that  typical working students are young, relying on the fact that only the information related to the connection of  $\mathit{WStudent}$ and $\mathit{Young}$ (and, in particular, only the defeasible inclusions occurring on the routes connecting $\mathit{WStudent}$ and $\mathit{Young}$ in the corresponding net) are used in the rational closure construction for answering the query.

Another approach which deals with the above problem of inheritance blocking has been proposed by Bonatti et al. in \cite{bonattiAIJ15},
where the logic ${\cal DL}^N$ captures a form of  ``inheritance with overriding": a defeasible inclusion is inherited by a more specific class if it is not overridden by more specific (conflicting) properties.  In Example  \ref{exa:weak}, our construction behaves differently from ${\cal DL}^N$, as in ${\cal DL}^N$
%While in the logic ${\cal DL}^N$ \cite{bonattiAIJ15}, given the knowledge base $K'$,
 the concept $\mathit{WStudent}$ has an inconsistent prototype, as working students inherit two conflicting properties by superclasses:  the property of  students of non paying taxes and the property of workers of paying taxes.
%and $N \mathit{WStudent} \sqsubseteq \bot$  is derivable, 
Instead, in the skeptical closure one cannot conclude that { $\mathit{\tip({WStudent})}$ $ \mathit{\sqsubseteq \bot}$} and,
using the terminology in \cite{bonattiAIJ15}, the conflict is ``silently removed". In this respect, the  skeptical closure appears to be weaker than ${\cal DL}^N$,
although it shares with ${\cal DL}^N$ (and with the lexicographic closure) a notion of overriding. 

Bozzato et al. in \cite{Bozzato2018}
present an extension of the CKR framework in which defeasible axioms are allowed in the global context 
%and exceptions can be handled by overriding and have to be justified in terms of semantic consequence.
and can be overridden by knowledge in a local context. Exceptions have to be justified in terms of semantic consequence.
A translation of extended CHRs (with  knowledge bases in ${\cal SROIQ}$-RL) into Datalog programs under the answer set semantics is also defined.

%In \cite{Bozzato14} the CKR framework  is presented; it is
%based on {\em SROIQ-RL},
%allows for defeasible axioms with local exceptions and exploits a translation to Datalog with negation.
%It is shown that instance checking in CKR reduces to (cautious) inference under the answer set semantics.

Concerning the multipreference semantics  introduced in \cite{GliozziAIIA2016} (and further refined in \cite{MultipreferenceArxiv2018}) to provide a semantic strengthening of the  rational closure,
we have shown in \cite{MultipreferenceArxiv2018} that the MP-semantics, a variant of Lehmann's lexicographic closure which does not take into account the number of defaults 
within the same rank, but only their subset inclusion (as recalled in Section \ref{sez:lex_closure}), provides a sound approximation of the multipreference semantics.
In this paper we have given a semantic characterization of the MP-closure by bi-preference minimal entailment. As a consequence, BP-minimal entailment is weaker that the multipreference semantics and, furthermore, 
the skeptical closure introduced in Section  \ref{sec:Skeptical_closure} is still a sound, weaker, approximation for the multipreference semantics in \cite{MultipreferenceArxiv2018}.

The relationships among the above variants of rational closure for DLs and the notions of rational closure for DLs developed in the contexts of fuzzy  logic \cite{CasiniStracciaLPAR2013} and probabilistic logics \cite{Lukasiewicz08} %will be subject of future work. 
are worth of being investigated.  %to be investigated as well.
As it has been show in \cite{dubois-prade} for the propositional logic case, KLM preferential logics and the rational closure \cite{KrausLehmannMagidor:90,whatdoes}, the probabilistic approach \cite{Adams:75}, the system Z \cite{PearlTARK90} and the possibilistic approach  \cite{BenferhatKR92,dubois-prade} %deDupinJAR08} 
are all related with each other, and similar relations might be expected to hold for the non-monotonic extensions of description logics as well.
Although the skeptical closure has been defined based on the preferential extension of $\alc$, the same construction could be adopted for more expressive description logics, provided the rational closure can be consistently defined \cite{GiordanoFI2018}, as well as for  the propositional case. \normalcolor

\medskip
{\bf Acknowledgement:} This research is partially supported by INDAM-GNCS Project 2018 ``Metodi di prova orientati al ragionamento automatico per logiche non-classiche".

%\bibliography{biblioMultipreferenze}

\end{document}